%% file: neural-path-features.tex
\documentclass{article}
\usepackage[final]{neurips_2020}
\usepackage[utf8]{inputenc} 
\usepackage[T1]{fontenc}    
\usepackage{hyperref}       
\usepackage{url}            
\usepackage{booktabs}    
\usepackage{amsfonts}    
\usepackage{nicefrac}      
\usepackage{microtype}   
\usepackage{bbm}
\input{pack.tex}

\usepackage{hyperref}
\usepackage{thmtools}
\usepackage{nameref}

\usepackage[nameinlink,noabbrev,capitalize]{cleveref}
\crefname{assumption}{assumption}{assumptions}
\crefname{lemma}{lemma}{lemmas}
\Crefname{lemma}{Lemma}{Lemmas}
\crefname{thm}{theorem}{theorems}
\Crefname{thm}{Theorem}{Theorems}
\crefname{proposition}{proposition}{propositions}

\usepackage{caption}
\usepackage{wrapfig,lipsum}
\usepackage[linewidth=1.2pt,linecolor=red]{mdframed}

\usepackage{comment}
\usepackage{cancel}
\usepackage{changepage}
\usepackage{bbm}
\usepackage{pgfplots}
\usepackage{filecontents}
\usepackage{placeins}

\title{Neural Path Features and Neural Path Kernel : Understanding the role of gates in deep learning}
\author{Chandrashekar Lakshminarayanan${}^*$ and Amit Vikram Singh\thanks{Equal Contribution}, \\ Indian Institute of Technology Palakkad\\\texttt{chandru@iitpkd.ac.in, amitkvikram@gmail.com}}
\begin{document}
\maketitle
\begin{abstract}
Rectified linear unit (ReLU) activations can also be thought of as \emph{gates}, which, either pass or stop their pre-activation input when they are \emph{on} (when the pre-activation input is positive) or \emph{off} (when the pre-activation input is negative) respectively. A deep neural network (DNN) with ReLU activations has many gates, and the {on/off} status of each gate changes across input examples as well as network weights. For a given input example, only a subset of gates are \emph{active}, i.e., on, and the sub-network of weights connected to these active gates is responsible for producing the output. At randomised initialisation, the active sub-network corresponding to a given input example is random. During training, as the weights are learnt, the active sub-networks are also learnt, and could hold valuable information. 

In this paper, we analytically characterise the role of gates and active sub-networks in deep learning. To this end, we encode the {on/off} state of the gates for a given input in a novel \emph{neural path feature} (NPF), and the weights of the DNN are encoded in a novel \emph{neural path value} (NPV). Further, we show that the output of network is indeed the inner product of NPF and NPV.  The main result of the paper shows that the \emph{neural path kernel} associated with the NPF is a fundamental quantity that characterises the information stored in the gates of a DNN. We show via experiments (on MNIST and CIFAR-10) that in standard DNNs with ReLU activations NPFs are learnt during training and such learning is key for generalisation. Furthermore, NPFs and NPVs can be learnt in two separate networks and such learning also generalises well in experiments. In our experiments, we observe that almost all the information learnt by a DNN with ReLU activations is stored in the gates - a novel observation that underscores the need to investigate the role of the gates in DNNs.
\end{abstract}

\section{Introduction}
We consider deep neural networks (DNNs) with rectified linear unit (ReLU) activations. A special property of the ReLU activation (denoted by $\chi$) is that its output can be written as a product of its pre-activation input, say $q\in\R$ and a gating signal, $G(q)=\mathbbm{1}_{\{q>0\}}$, i.e., $\chi(q)=q\cdot G(q)$. While the weights of a DNN remain the same across input examples, the $1/0$ state of the gates (or simply gates) change across input examples. For each input example, there is a corresponding \emph{active sub-network} consisting of those gates which are $1$, and the weights which pass through such gates. This active sub-network can be said to hold the memory for a given input, i.e., only those weights that pass through such active gates contribute to the output. In this viewpoint, at random initialisation of the weights, for a given input example, a random sub-network is active and produces a random output.  However, as the weights change during training (say using gradient descent), the gates change, and hence the active sub-networks corresponding to the various input examples also change. At the end of training, for each input example, there is a learnt active sub-network, and produces a learnt output. Thus, the gates of a trained DNN could potentially contain valuable information. 

We focus on DNNs with ReLU activations. The goal and claims in this paper are stated below.\\
\resizebox{\columnwidth}{!}{
\begin{tabular}{ll}
Goal &: \emph{To study the role of the gates in DNNs trained with gradient descent (GD).}\\
Claim I (\Cref{sec:infomeasure})  &: \emph{Active sub-networks are fundamental entities.}  \\
Claim II (\Cref{sec:experiments})	&: \emph{Learning of the active sub-networks improves generalisation.}
\end{tabular}
}
Before we discuss the claims in terms of our novel technical contributions in \Cref{sec:contrib}, we present the background of  \emph{neural tangent} framework in \Cref{sec:background}.

\textbf{Notation:} We denote the set $\{1,\ldots, n\}$ by $[n]$. For $x,y\in\R^m$, $\ip{x,y}=x^\top y$. The dataset is denoted by $(x_s,y_s)_{s=1}^n\in\R^{\din}\times \R$. For an input $x\in\R^{\din}$, the output of the DNN  is denoted by $\hat{y}_{\Theta}(x)\in\R$, where $\Theta\in\R^{\dnet}$ are the weights. We use $\theta\in\Theta$ to denote a single arbitrary weight, and $\partial_{\theta}(\cdot)$ to denote $\frac{\partial (\cdot)}{\partial \theta}$. We use $\nabla_{\Theta}(\cdot)$ to denote the gradient of $(\cdot)$ with respect to the network weights. We use vectorised notations $y=(y_s,s\in[n]), \hat{y}_{\Theta}=\left(\hat{y}_{\Theta}(x_s), s\in[n]\right)\in\R^n$ for the true and predicted outputs and $e_{t}= (\hat{y}_{\Theta_t}-y)\in\R^n$ for the error in the prediction.

\subsection{Background: Neural Tangent Feature and Neural Tangent Kernel}\label{sec:background}
The neural tangent machinery was developed in some of the recent works [\citenum{ntk, arora2019exact,cao2019generalization,dudnn}] to understand optimisation and generalisation in DNNs trained using GD. For an input $x\in\R^{\din}$, the \emph{neural tangent feature} (NTF) is given by $\psi_{x,\Theta}=\nabla_{\Theta}\hat{y}_{\Theta}(x)\in\R^{\dnet}$, i.e.,  the gradient of the network output with respect to its weights. The \emph{neural tangent kernel} (NTK) matrix $K_{\Theta}$ on the dataset is the $n\times n$ Gram matrix of the NTFs of the input examples, and is given by $K_{\Theta}(s,s')=\ip{\psi_{x_s,\Theta},\psi_{x_{s'},\Theta}}, s,s'\in[n]$. 
\begin{proposition}[\textbf{Lemma 3.1} \cite{arora2019exact}]\label{prop:basic}
Consider the GD procedure to minimise the  squared loss $L(\Theta)=\frac{1}{2}\sum_{s=1}^n \left(\hat{y}_{\Theta}(x_s)-y_s\right)^2$ with infinitesimally small step-size given by $\dot{\Theta}_t=-\nabla_{\Theta}L_{\Theta_t}$. It follows that the dynamics of the error term can be written as $\dot{e}_t=-K_{\Theta_t} e_t$. 
\end{proposition}

\textbf{Prior works} [\citenum{ntk,dudnn,arora2019exact,cao2019generalization}] have studied DNNs trained using GD in the so called `NTK regime', which occurs (under appropriate randomised initialisation) when the width of the DNN approaches infinity. The characterising property of the NTK regime is that as $w\ra\infty$, $K_{\Theta_0}\ra K^{(d)}$, and $K_{\Theta_t}\approx K_{\Theta_0}$, where $K^{(d)}$ (see \eqref{eq:ntkold} in \Cref{sec:kd}) is a deterministic matrix whose superscript $(d)$ denotes the depth of the DNN. \cite{arora2019exact} showed that an infinite width DNN trained using GD is equivalent to a kernel method with the limiting NTK matrix $K^{(d)}$ (and hence enjoys the generalisation ability of the limiting NTK matrix $K^{(d)}$). Further, \cite{arora2019exact} proposed a pure kernel method based on what they call the CNTK, which is the limiting NTK matrix $K^{(d)}$ of an infinite width convolutional neural network (CNN). \cite{cao2019generalization} showed a generalisation bound of the form $\tilde{\mathcal{O}}\left(d\cdot\sqrt{y^\top {\left(K^{(d)}\right)}^{-1} y/n}\right)$\footnote{$a_t=\mathcal{O}(b_t)$ if $\lim\sup_{t\ra\infty}|a_t/b_t|<\infty$, and $\tilde{\mathcal{O}}(\cdot)$ is used to hide logarithmic factors in $\mathcal{O}(\cdot)$.} in the NTK regime. 

\textbf{Open Question:} \cite{arora2019exact} reported a $5\% - 6\%$ performance gain of finite width CNN (not operating in the NTK regime) over the exact CNTK corresponding to infinite width CNN, and inferred that the study of DNNs in the NTK regime cannot fully explain the success of practical neural networks yet. Can we explain the reason for the performance gain of CNNs over CNTK?

\subsection{Our Contributions}\label{sec:contrib}
To the best of our knowledge, we are the first to analytically characterise the role played by the gates and the active sub-networks in deep learning as presented in the `Claims I and II'.  The key contributions can be arranged into three landmarks as described below.

$\bullet$ The first step involves breaking a DNN into individual paths, and each path again into gates and weights.  To this end, we encode the states of the gates in a novel \emph{neural path feature} (NPF) and the weights in a novel \emph{neural path value} (NPV) and express the output of the DNN as an inner product of NPF and NPV (see \Cref{sec:path}). In contrast to NTF/NTK which are \emph{first-order} quantities (based on derivatives with respect to the weights), NPF and NPV are \emph{zeroth-order} quantities. The kernel matrix associated to the NPFs namely the \emph{neural path kernel} (NPK) matrix $H_{\Theta}\in\R^{n\times n}$ has a special structure, i.e., it can be written as a \emph{Hadamard} product of the input Gram matrix, and a correlation matrix $\Lambda_{\Theta}\in\R^{n\times n}$, where $\Lambda_{\Theta}(s,s')$ is proportional to the total number of paths in the sub-network that is active for both input examples $s,s'\in[n]$. With the $\Lambda_{\Theta}$ matrix we reach our first landmark.

$\bullet$  Second step is to characterise performance of the gates and the active sub-networks in a `stand alone' manner. To this end, we consider a new idealised setting namely the fixed NPF (FNPF) setting, wherein, the NPFs are fixed (i.e., held constant) and only the NPV is learnt via gradient descent. In this setting, we show that (see \Cref{th:main}), in the limit of infinite width and under randomised initialisation the NTK converges to a matrix $K^{(d)}_{\text{FNPF}}=\text{constant} \times H_{\text{FNPF}}$, where $H_{\text{FNPF}}\in\R^{n\times n}$ is the NPK matrix corresponding to the fixed NPFs. $K^{(d)}$ matrix of \cite{ntk,arora2019exact,cao2019generalization} becomes the $K^{(d)}_{\text{FNPF}}$ matrix in the FNPF setting,  wherein, we initialise the NPV statistically independent of the fixed NPFs (see \Cref{th:main}). With \Cref{th:main}, we reach our second landmark, i.e. we justify ``Claim I'', that active sub-networks are fundamental entities, which follows from the fact that $H_{\text{FNPF}}=\Sigma\odot \Lambda_{\text{FNPF}}$, where $\Lambda_{\text{FNFP}}$ corresponds to the fixed NPFs.

$\bullet$ Third step is to show experimentally that sub-network learning happens in practice (see \Cref{sec:experiments}). We show that in finite width DNNs with ReLU activations, NPFs are learnt continuously during training, and such learning improves generalisation. We observe that fixed NPFs obtained from the initial stages of training generalise poorly than CNTK (of \cite{arora2019exact}), whereas, fixed NPFs obtained from later stages of training generalise better than CNTK and generalise as well as standard DNNs with ReLU. This throws light on the open question in \Cref{sec:background}, i.e., the difference between the NTK regime and the finite width DNNs is perhaps due to NPF learning. In finite width DNNs, NPFs are learnt during training and in the NTK regime no such feature learning happens during training. Since the NPFs completely encode the information pertaining to the active sub-networks, we complete our final landmark namely  justification of ``Claim II''. 

\section{Neural Path Feature and Kernel: Encoding Gating Information}\label{sec:path}
First step in understanding the role of the gates is to explicitly \emph{encode} the $1/0$ states of the gates. The gating property of the ReLU activation allows us to express the output of the DNN as a summation of the contribution of the individual paths, and paves a natural way to encode the $1/0$ states of the gates \emph{without loss of information}. The contribution of a path is the product of the signal at its input node, the weights in the path and the gates in the path. For an input $x\in\R^{\din}$, and weights $\Theta\in\R^{\dnet}$, 
we encode the gating information in a novel \emph{neural path feature} (NPF), $\phi_{x,\Theta}\in\R^P$ and encode the weights in a novel \emph{neural path value} (NPV) $v_{\Theta}\in\R^P$, where $P$ is the total number of paths. 
The NPF co-ordinate of a path is the product of the signal at its input node and the gates in the path. The NPV co-ordinate of a path is the product of the weights in the paths. The output is then given by
\begin{align}\label{eq:npfnpv}
\hat{y}_{\Theta}(x)=\ip{\phi_{x,\Theta}, v_{\Theta}},
\end{align}
where $\phi_{x,\Theta}$ can be seen as the \emph{hidden features} which along with $v_{\Theta}$ are learnt by gradient descent.  
\subsection{Paths, Neural Path Feature, Neural Path Value and Network Output}
We consider fully-connected DNNs with `$w$' hidden units per layer and `$d-1$' hidden layers. $\Theta\in\R^{\dnet}$ are the network weights, where $\dnet=\din w+(d-2)w^2+w$. The information flow is shown in \Cref{tb:basic}, where
$\Theta(i,j,l)$ is the weight connecting the $j^{th}$ hidden unit of layer $l-1$ to the $i^{th}$ hidden unit of layer $l$. Further, $\Theta(\cdot,\cdot,1)\in\R^{w\times \din}, \Theta(\cdot,\cdot,l)\in\R^{w\times w},\forall l\in\{2,\ldots,d-1\}, \Theta(\cdot,\cdot,d)\in\R^{1\times w}$.
\begin{table}[h]
\centering
\begin{tabular}{|l l lll|}\hline
Input Layer&: &$z_{x,\Theta}(0)$ &$=$ &$x$ \\
Pre-Activation Input&: & $q_{x,\Theta}(i,l)$& $=$ & $\sum_{j} \Theta(i,j,l)\cdot z_{x,\Theta}(j,l-1)$\\
Gating Values&: &$G_{x,\Theta}(i,l)$& $=$ & $\mathbbm{1}_{\{q_{x,\Theta}(i,l)>0\}}$\\
Hidden Layer Output&: &$z_{x,\Theta}(i,l)$ & $=$ & $q_{x,\Theta}(i,l)\cdot G_{x,\Theta}(i,l)$ \\
Final Output&: & $\hat{y}_{\Theta}(x)$ & $=$ & $\sum_{j\in[w]} \Theta(1,j,d-1)\cdot z_{x,\Theta}(j,d-1)$\\\hline
\end{tabular}
\caption{Here, $l\in[d-1],i\in[w]$, $j\in[\din]$ for $l=1$ and $j\in[w]$ for $l=2,\ldots,d-1$.} 
\label{tb:basic}
\end{table}

\textbf{Paths:} A path starts from an input node, passes through exactly one weight and one hidden node in each layer and ends at the output node. We have a total of $P=\din w^{(d-1)}$ paths. We assume that there is a natural enumeration of the paths, and denote the set of all paths by $[P]$. Let $\I_{l}\colon [P]\ra [w],l=1,\ldots,d-1$ provide the index of the hidden unit through which a path $p$ passes in layer $l$, and $\I_{0}\colon [P]\ra [\din]$ provides the input node, and $\I_{d}(p)=1,\forall p\in[P]$.
\begin{definition}\label{def:nps} Let $x\in\R^{\din}$ be the input to the DNN. For this input, 

(i) The activity of a path $p$ is given by : $A_{\Theta}(x,p)\stackrel{def}{=}\Pi_{l=1}^{d-1} G_{x,\Theta}(\I_l(p),l)$.

(ii) The {neural path feature} (NPF) is given by :  $\phi_{x,\Theta}\stackrel{def}=\left(x(\I_0(p))A_{\Theta}(x,p) ,p\in[P]\right)\in\R^P$. 

(iii) The {neural path value} (NPV) is given by : $v_{\Theta}\stackrel{def}=\left(\Pi_{l=1}^d \Theta(\I_l(p),\I_{l-1}(p),l),p\in[P]\right)\in\R^P$.
\end{definition}
\textbf{Remark:} A path $p$ is active if all the gates in the paths are on.
\begin{proposition}\label{prop:zero}  The output of the network can be written as an inner product of the NPF and NPV, i.e., 
$\hat{y}_{\Theta}(x)=\ip{\phi_{x,\Theta},v_{\Theta}}=\sum_{p\in [P]}x(\I_0(p))A_{\Theta}(x,p)v_{\Theta}(p)$.
\end{proposition}

\FloatBarrier
\begin{figure}[h]
\centering
\resizebox{\columnwidth}{!}{
\includegraphics[scale=1]{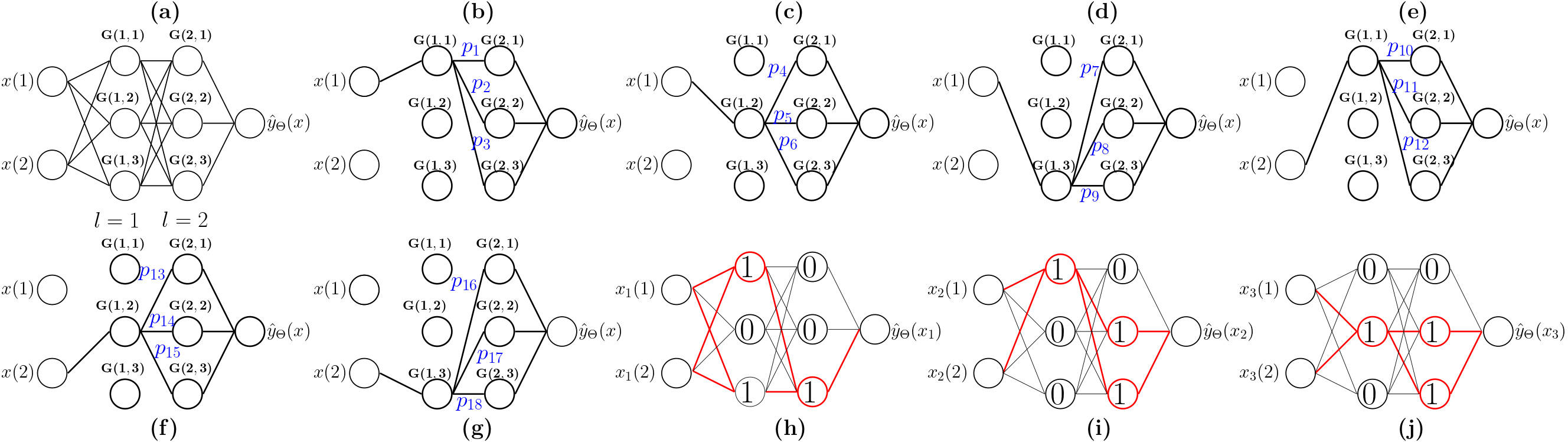}
}
\centering
\begin{minipage}{0.64\textwidth}
\resizebox{\columnwidth}{!}{
$\phi_{x_1}=[0, 0, x_1(1),0,0, 0, 0,0, x_1(1), 0, 0, x_1(2),0, 0, 0,0, 0, x_1(2)]^\top$
}
\\
\resizebox{\columnwidth}{!}{
$\phi_{x_2}=[0, x_2(1), x_2(1),0,0, 0, 0,0,0, 0, x_2(2), x_2(2),0,0, 0, 0,0,0 ]^\top$
}
\\
\resizebox{\columnwidth}{!}{
$\phi_{x_3}=[0, 0,0,0, x_3(1),x_3(1),0,0, 0,0, 0,0,0, x_3(2),x_3(2),0,0, 0 ]^\top$
}
\end{minipage}
\begin{minipage}{0.25\textwidth}
,\,$\Lambda=\left[\begin{matrix} 
2 &1& 0 \\
1 &2& 0\\
0 &0& 2
\end{matrix}\right]$
\end{minipage}
\caption{A toy illustration of gates, paths and active sub-networks. The cartoon (\textbf{a}) in the top left corner shows a DNN with $2$ hidden layers, $6$ ReLU gates $G(l,i),l=1,2,i=1,2,3$, $2$ input nodes $x(1)$ and $x(2)$ and an output node $\hat{y}_{\Theta}(x)$. Cartoons (\textbf{b}) to (\textbf{g}) show the enumeration of the paths $p_1,\ldots, p_{18}$. Cartoons (\textbf{h}), (\textbf{i}) and (\textbf{j}) show hypothetical gates for $3$ different hypothetical input examples $\{x_s\}_{s=1}^3 \in\R^2$. In each of the cartoons (\textbf{h}), (\textbf{i}) and (\textbf{j}), the $1/0$ inside the circles denotes the on/off state of the gates, and the bold paths/gates shown in red colour constitute the active sub-network for that particular input example. The NPFs are given by $\phi_{x}=[x(1)A(x,p_1),\ldots,x(1)A(x,p_{9}),x(2)A(x,p_{10}),\ldots,x(2)A(x,p_{18})]^\top$. Here, $\Lambda(1,2)=1$ because paths $p_3$ and $p_{12}$ are both active for input examples $x_1$ and $x_2$ and the input dimension is $2$.}
\label{fig:npkexample}
\end{figure}
\subsection{Neural Path Kernel : Similarity based on active sub-networks}
\begin{definition}\label{def:lambda}
 For input examples $s,s'\in[n]$, define $Act_{\Theta}(s,s')\stackrel{def}=\{p\in[P]\colon A_{\Theta}(x_s,p)= A_{\Theta}(x_{s'},p)=1\}$ to be the set of `active' paths for both $s,s'$  and $\Lambda_{\Theta}(s,s')\stackrel{def}=\frac{|Act_{\Theta}(s,s')|}{\din}$.
\end{definition}
\textbf{Remark:} Owing to the symmetry of a DNN, the same number of active paths start from any fixed input node. In \Cref{def:lambda}, $\Lambda_{\Theta}$ measures the size of the active sub-network as the total number of active paths starting from any fixed input node. For examples $s,s'\in[n],s\neq s'$, $\Lambda_{\Theta}(s,s)$ is equal to the size of the sub-network active for $s$, and $\Lambda_{\Theta}(s,s')$ is equal to the size of the sub-network active for both $s$ and $s'$. For an illustration of NPFs and $\Lambda$ please see \Cref{fig:npkexample}.
\begin{lemma}\label{lm:npk}
Let $H_{\Theta}\in\R^{n\times n}$ be the NPK matrix, whose entries are given by $H_{\Theta}(s,s')\stackrel{def}{=}\ip{\phi_{x_s,\Theta},\phi_{x_{s'},\Theta}},s,s'\in[n]$. Let $\Sigma\in\R^{n\times n}$ be the input Gram matrix with entires $\Sigma(s,s')=\ip{x_s,x_{s'}},s,s'\in[n]$. It follows that $H_{\Theta}= \Sigma\odot\Lambda_{\Theta}$, where $\odot$ is  the Hadamard product.
\end{lemma}

\section{Dynamics of Gradient Descent with NPF and NPV Learning}\label{sec:gatedyna}
In \Cref{sec:path}, we mentioned that during gradient descent, the DNN is learning a relation $\hat{y}_{\Theta}(x)=\ip{\phi_{x,\Theta}, v_{\Theta}}$, i.e., both the NPFs and the NPV are learnt during gradient descent. In this section, we connect the newly defined quantities, i.e, NPFs and NPV to the NTK matrix $K_{\Theta}$ (see \Cref{prop:ntknew}), and re-write the gradient descent dynamics in \Cref{prop:dnnhard}. In what follows, we use $\Phi_{\Theta}=(\phi_{x_s,\Theta},s\in[n])\in\R^{P\times n}$ to denote the NPF matrix.
\subsection{Dynamics of NPFs and NPV}
\begin{definition}\label{def:npvgrad}
The gradient of the NPV of a path $p$ is defined as $\varphi_{p,\Theta}\stackrel{def}=(\partial_{\theta}v_{\Theta}(p), \theta \in\Theta)\in\R^{\dnet}$.
\end{definition}
\textbf{Remark:} The change of the NPV is given by $\dot{v}_{\Theta_t}(p)=\ip{\varphi_{p,\Theta_t},\dot{\Theta}_t}$, where $\dot{\Theta}_t$ is the change of the weights. We now collect the gradients $\varphi_{p,\Theta}$ of all the paths to define a \emph{value tangent kernel} (VTK). 
\begin{definition}
Let $\nabla_{\Theta}v_{\Theta}$ be a $\dnet\times P$ matrix of NPV derivatives given by $\nabla_{\Theta}v_{\Theta}=(\varphi_{p,\Theta},p\in[P])$. Define the VTK to be the $P\times P$ matrix given by $\V_{\Theta}=(\nabla_{\Theta}v_{\Theta})^\top(\nabla_{\Theta}v_{\Theta})$.
\end{definition}
\textbf{Remark:} An important point to note here is that the VTK is a quantity that is dependent only on the weights. To appreciate the same, consider a deep linear network (DLN) [\citenum{shamir,dudln}] which has identity activations, i.e., all the gates are $1$ for all inputs, and weights. For a DLN and DNN with identical network architecture (i.e., $w$ and $d$), and identical weights, $\V_{\Theta}$ is also identical. Thus, $\V_{\Theta}$ is the gradient based information that excludes the gating information.

The NPFs changes at those time instants when any one of the gates switches from $1$ to $0$ or from $0$ to $1$. In the time between two such switching instances, NPFs of all the input examples in the dataset remain the same, and between successive switching instances,  the NPF of at least one of the input example in the dataset changes. In what follows, in \Cref{prop:dnnhard} we re-write \Cref{prop:basic} taking into account the switching instances which we define in \Cref{def:switch}.
\begin{definition}\label{def:switch}
Define a sequence of monotonically increasing time instants $\{T_{i}\}_{i=0}^\infty$ (with $T_0=0$) to be `switching' instants if $\phi_{x_s,\Theta_t}=\phi_{x_s,\Theta_{T_i}},\forall s\in[n],\forall t\in[T_{i},T_{i+1}), i\geq 0$, and  $\forall i\geq 0$, there exists $s(i)\in[n]$ such that $\phi_{x_{s(i)},\Theta_{T_i}}\neq \phi_{x_{s(i)},\Theta_{T_{i+1}}}$.
\end{definition}
\subsection{Re-writing Gradient Descent Dynamics}
\begin{proposition}\label{prop:ntknew}
The NTK is given by $K_{\Theta}=\Phi^\top_{\Theta}\V_{\Theta}\Phi_{\Theta}$.
\end{proposition}
\textbf{Remark:} $K_{\Theta_t}$ changes during training (i) continuously at all $t\geq 0$ due to $\V_{\Theta_t}$, and (ii) at switching instants $T_{i},i=0,\ldots,\infty$ due to the change in $\Phi_{\Theta_{T_i}}$. We now describe the gradient descent dynamics taking into the dynamics of the NPV and the NPFs.
\begin{proposition}\label{prop:dnnhard}
Let $\{T_i\}_{i=0}^\infty$ be as in \Cref{def:switch}. For $t\in[T_{i},T_{i+1})$ and small step-size of GD:
\FloatBarrier
\begin{table}[h]
\centering
\begin{tabular}{ l c l l l }
Weights Dynamics &:  & $\dot{\Theta}_t$&$=$&$-\sum_{s=1}^n \psi_{x_s,\Theta_t}e_t(s)$\\
NPV Dynamics&: & $\dot{v}_{\Theta_t}(p)$&$=$&$\ip{\varphi_{p,\Theta_t},\dot{\Theta}_t},\forall p\in[P]$\\
Error Dynamics&: & $\dot{e}_t$&$=$&$-K_{\Theta_t}e_t$, where $K_{\Theta_t}=\Phi^\top_{\Theta_{T_i}}\V_{\Theta_t}\Phi_{\Theta_{T_i}}$\\
\end{tabular}
\end{table}
\end{proposition}
\begin{proposition}\label{prop:condition}
Let the maximum and minimum eigenvalues of a real symmetric matrix $A$ be denoted by $\rho_{\max}(A)$ and $\rho_{\min}(A)$. Then, $\rho_{\min}(K_{\Theta})\leq \rho_{\min}(H_{\Theta})\rho_{\max}\left(\V_{\Theta}\right)$.
\end{proposition}
\textbf{Remark:} For the NTK to be well conditioned, it is necessary for the NPK to be well conditioned. This is intuitive, in that, the closer two inputs are, the closer are their NPFs, and it is harder to train the network to produce arbitrarily different outputs for such inputs that are very close to one another.

\section{Deep Gated Networks: Decoupling Neural Path Feature and Value}\label{sec:decoupled}
The next step towards our goal of understanding the role of the gates (and gate dynamics) is the separation of the gates (i.e., the NPFs) from the weights  (i.e., the NPV).  This is achieved by a deep gated network (DGN) having two networks of identical architecture namely i) a feature network parameterised by $\Tf\in\R^{\dfnet}$, that holds gating information, and hence the NPFs and ii) a value network that holds the NPVs parameterised by $\Tv\in\R^{\dvnet}$. As shown in \Cref{fig:dgn}, the gates/NPFs are generated in the feature network and are used in the value network. In what follows, we let $\Tdgn=(\Tf,\Tv)\in\R^{\dfnet+\dvnet}$ to denote the combined parameters of a DGN. 
\begin{figure}[t] 
\begin{minipage}{0.79\columnwidth}
\resizebox{\columnwidth}{!}{
\begin{tabular}{|  l | l |}\hline
 Feature Network (NPF)& Value Network (NPV)\\
 $z^{\text{f}}_{x}(0)=x$ &$z^{\text{v}}_{x}(0)=x$ \\
$q^{\text{f}}_{x}(i,l)=\sum_{j} \Tf(i,j,l)\cdot z_{x}(j,l-1)$ & $q^{\text{v}}_{x}(i,l)=\sum_{j} \Tv(i,j,l)\cdot z^{\text{v}}_{x}(j,l-1)$\\
$z^{\text{f}}_{x}(i,l)= q^{\text{f}}_{x}(i,l)\cdot\mathbbm{1}_{\{q^{\text{f}}_{x}(i,l)>0\}}$& $z^{\text{v}}_{x}(i,l)= q^{\text{v}}_{x}(i,l)\cdot G_{x}(i,l)$ \\\hline
 \multicolumn{2}{|c|}{Output: $\hat{y}_{\Tdgn}(x)= \sum_{j} \Tv(1,j,l)\cdot z^{\text{v}}_{x}(j,d-1)$}\\\hline
\multicolumn{2}{|l|}{{Hard ReLU: $G_{x}(i,l)=\mathbbm{1}_{\{q^{\text{f}}_{x}(i,l)>0\}}$ or Soft-ReLU: $G_{x}(i,l)={1}/{\left(1+\exp(-\beta\cdot q^{\text{f}}_{x}(i,l)>0)\right)} $}}\\\hline
\end{tabular}
}
\end{minipage}
\begin{minipage}{0.20\columnwidth}
\resizebox{\columnwidth}{!}{
\includegraphics[scale=0.4]{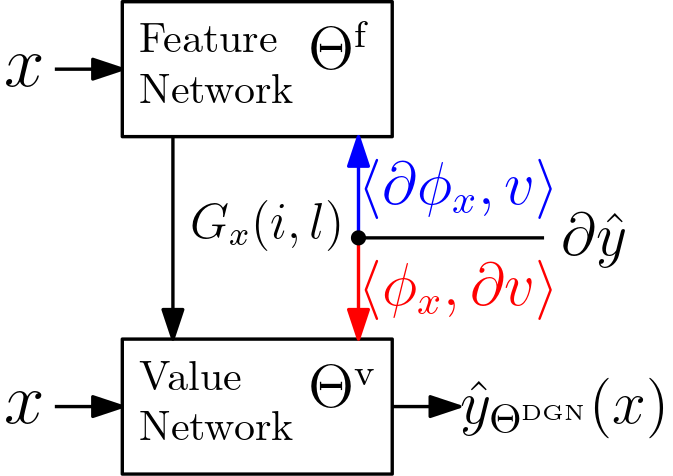}
}
\end{minipage}
\caption{Deep gated network (DGN) setup.  The pre-activations $q^{\text{f}}_{x}(i,l)$ from the feature network are used to derive the gating values $G_{x}(i,l)$. The range of the indices $i,j,l$ is the same as in \Cref{tb:basic}. }
\label{fig:dgn}
\end{figure}

\textbf{Regimes of a DGN:} We can configure the DGN in four different \emph{regimes} by controlling (i)  the trainability of $\Tf$, and (ii) the initialisation of $\Tf_0$. By setting $\Tf$ to be \emph{non-trainable/trainable} we can compare fixed NPFs and NPFs that change during training. By setting $\Tf_0$ to be \emph{random/pre-trained} we can compare random NPFs and learnt NPFs. By setting $\Tf_0=\Tv_0$ we can mimic the initialisation of a standard DNN with ReLU. The four regimes of a DGN are described below (note that in all the regimes $\Tv_0$ is randomly initialised,  $\Tv$ is trainable and $\hat{y}_{\Tdgn}$ is the output node).

1. Decoupled Learning of NPF (\textbf{DLNPF}): Here, $\Tf$ is trainable, and hence the NPFs are learnt in a decoupled manner (as opposed to the standard DNN with ReLU where a single parameter is responsible for learning NPFs and NPV). Here, soft-ReLU gate with $\beta>0$ is used to ensure gradient flow via feature network. $\Tf_0,\Tv_0$ are initialised at random and are statistically independent.

2. Fixed Learnt NPF (\textbf{FLNPF}): Here $\Tf$ is non-trainable, and $\Tf_0$ copied from a  pre-trained DNN with ReLU (NPFs are learnt). $\Tv_0$ are initialised at random and is statistically independent of $\Tf_0$.

3. Fixed Random NPF with Independent Initialisation (\textbf{FRNPF-II}): Here,  $\Tf_0,\Tv_0$ are initialised at random are statistically independent. Also, $\Tf$ is non-trainable, i.e.,  the NPFs are random and fixed.

4. Fixed Random NPF with Dependent Initialisation (\textbf{FRNPF-DI}):  Here, the initialisation mimics standard DNN with ReLU, i.e., $\Tf_0=\Tv_0$ are initialised at random, and $\Tf$ is non-trainable.

\textbf{Remark:} The DGN and its regimes are idealised models to understand the role of the gates, and not alternate proposals to replace standard DNNs with ReLU activations. 
\begin{proposition}[Gradient Dynamics in a DGN]\label{prop:dgn} Let $\psif_{x,\Tdgn}\stackrel{def}=\nabla_{\Tf}\hat{y}_{\Tdgn}(x) \in \R^{\dfnet}$, $\psiv_{x,\Tdgn}\stackrel{def}=\nabla_{\Tv}\hat{y}_{\Tdgn}(x) \in \R^{\dvnet}$. Let $K^{\text{v}}_{\Tdgn}$ and $K^{\text{f}}_{\Tdgn}$ be $n\times n$ matrices with entries $K^{\text{v}}_{\Tdgn}(s,s')=\ip{\psiv_{x_s,\Tdgn},\psiv_{x_{s'},\Tdgn}}$ and $K^{\text{f}}_{\Tdgn}(s,s')=\ip{\psif_{x_s,\Tdgn},\psif_{x_{s'},\Tdgn}}$. For infinitesimally small step-size of GD, the error dynamics in a DGN (in the DLNPF and FNPF modes) is given by:
\FloatBarrier
\begin{table}[h]
\resizebox{\columnwidth}{!}{
\begin{tabular}{| l | l l l | l |}\hline
	Dynamics&		&&Decoupled Learning& Fixed NPF\\\hline
		&		&&& \\
Weight  & $\dot{\Theta}^{\text{v}}_t$&$=$&$-\sum_{s=1}^n \psiv_{x,\Tdgn_t}e_t(s),\dot{\Theta}^{\text{f}}_t=-\sum_{s=1}^n \psif_{x,\Tdgn_t}e_t(s)$ & $\dot{\Theta}^{\text{v}}_t$ same as (DLNFP),  $\dot{\Theta}^{\text{f}}_t=0$\\
NPF & $\dot{\phi}_{x_s,\Tf_t}(p)$&$=$&$x(\I_0(p))\sum_{\Tf\in\Tf}\partial_{\Tf}A_{\Tf_t}(x_s,p)\dot{\theta}^{\text{f}}_t,\forall p\in[P], s\in[n]$& $\dot{\phi}_{x_s,\Tf_t}(p)=0$\\
NPV & $\dot{v}_{\Tv_t}(p)$&$=$&$\sum_{\tv\in\Tv}\partial_{\tv}v_{\Tv_t}(p)\dot{\theta}^{\text{v}}_t,\forall p\in[P]$ & $\dot{v}_{\Tv_t}(p)$ same as DLNPF\\
Error & $\dot{e}_t$&$=$&$-\left(K^{\text{v}}_{\Tdgn}+ K^{\text{f}}_{\Tdgn}\right)e_t$  & $\dot{e}_t=-\left(K^{\text{v}}_{\Tdgn}\right)e_t$ \\\hline
\end{tabular}
}
\end{table}
\end{proposition}
\textbf{Remark:} The gradient dynamics in a DGN specified in \Cref{prop:dgn} is similar to the gradient dynamics in a DNN specified in \Cref{prop:dnnhard}. Important difference is that (in a DGN) the NTF $\psi_{x,\Theta}=(\psif_{x,\Theta},\psiv_{x,\Theta})\in\R^{\dfnet+\dvnet}$, wherein, $\psiv_{x,\Tdgn}\in\R^{\dvnet}$ and $\psif_{x,\Tdgn}\in\R^{\dfnet}$ flow through the value and feature networks respectively. Here, NPF learning is captured explicitly by $\psif$ and $\kf$. 

\section{Learning with Fixed NPFs: Role Of Active Sub-Networks}\label{sec:infomeasure}
We now show that in the fixed NPF regime, at randomised initialisation, $\text{NTK}=\text{const}\times\text{NPK}$. Due to the \emph{Hadamard} structure of the NPK (\Cref{lm:npk}) it follows that the active sub-networks are fundamental entities in DNNs (theoretically justifying  ``Claim I''). 
\begin{theorem} \label{th:main} Let $H_{\text{FNFP}}$ refer to $H_{\Tf_0}$. Let (i) $\Tv_0\in\R^{\dvnet}$ be statistically independent of  $\Tf_0\in\R^{\dfnet}$, and (ii) $\Tv_0$ be sampled i.i.d from symmetric Bernoulli over $\{-\sigma,+\sigma\}$. For $\sigma=\frac{\sigma'}{\sqrt{w}}$,  as $w\ra\infty$,  \begin{align*}K^{\text{v}}_{\Tdgn_0}\ra K^{(d)}_{\text{FNPF}} =d\cdot \sigma^{2(d-1)} \cdot H_{\text{FNPF}}\end{align*}
\end{theorem}
$\bullet$ \textbf{Statistical independence} of $\Tf_0$ and $\Tv_0$ assumed in \Cref{th:main} holds only for the three regimes namely DLNPF, FLNPF, FRNPF-II. In DNN with ReLU (and in FRNPF-DI) $\Tf_0=\Tv_0$, and hence the assumption in \Cref{th:main} does not capture the conditions at initialisation in a DNN with ReLU. However, it is important to note that the current state-of-the-art analysis for DNNs is in the infinite width regime [\citenum{ntk,arora2019exact,cao2019generalization}], wherein, the activations undergo only an order of $\frac{1}{\sqrt{w}}$ change during training. Since as $w\ra\infty$, $\frac{1}{\sqrt{w}}\ra 0$, assuming the NPFs (i.e., gates) to be fixed during training is not a strong one. With fixed NPFs, statistical independence of $\Tv_0$ is a natural choice. Also, we do not observe significant empirical difference between the FRNPF-DI and FRNPF-II regimes (see \Cref{sec:experiments}).

$\bullet$ \textbf{Role of active sub-networks:} Due to the statistical independence of $\Tf_0$ and $\Tv_0$, $K^{(d)}$ in prior works [\citenum{ntk,arora2019exact,cao2019generalization}] essentially becomes $K^{(d)}_{\text{FNPF}}$ in \Cref{th:main}.  From previous results [\citenum{arora2019exact,cao2019generalization}], it follows that as $w\ra\infty$, the optimisation and generalisation properties of the fixed NPF learner can be tied down to the infinite width NTK of the FNPF learner $K^{(d)}_{\text{FNPF}}$ and hence to $H_{\text{FNPF}}$ (treating $d\sigma^{2(d-1)}$ as a scaling factor).  We can further breakdown $H_{\text{FNPF}}=\Sigma\odot{\Lambda}_{\text{FNPF}}$, where $\Lambda_{\text{FNPF}}=\Lambda_{\Tf_0}$. This justifies ``Claim I'',  because $\Lambda_{\text{FNPF}}$ is the correlation matrix of the active sub-network overlaps. We also observe in our experiments that with the learnt NPFs (i.e., in the FLNPF regime), we can train the NPV without significant loss of performance, a fact that underscores the importance of the NPFs/gates.

 $\bullet$ \textbf{Scaling factor}  `$d$' is due to the `$d$' weights in a path and at $t=0$ the derivative of the value of a given path with respect any of its weights is $\sigma^{(d-1)}$. In the case of random NPFs obtained by initialising $\Tf_0$ at random (by sampling from a symmetric distribution), we expect $\frac{w}2$ gates to be `on' every layer, so $\sigma=\sqrt{\frac{2}{w}}$ is a normalising choice, in that, the diagonal entries of $\sigma^{2(d-1)}{\Lambda}_{\text{FNPF}}(s,s)\approx 1$ in this case.

$\bullet$ \Cref{th:main} can also be applied when the fixed NPFs are obtained from a finite width feature network by a `repetition trick' (see \Cref{sec:finite}). We  also apply \Cref{th:main} to understand the role of width and depth on a pure memorisation task  (see \Cref{sec:mem}).

\section{Experiments: Fixed NPFs, NPF Learning and Verification of Claim II}\label{sec:experiments} 
In this section, we experimentally investigate the role of gates and empirically justify ``Claim II'', i.e., learning of the active sub-networks improves generalisation.  In our framework, since the gates are encoded in the NPFs, we justify ``Claim II'' by showing that NPF learning improves generalisation. For this purpose, we make use of the DGN setup and its four different regimes.  We show that NPF learning also explains the performance gain of finite width CNN over the pure kernel method based on the exact infinite width CNTK. We now describe the setup and then discuss the results.
\subsection{Setup}
\textbf{Datasets:} We used standard datasets namely MNIST and CIFAR-10, with categorical cross entropy loss. We also used a `Binary'-MNIST dataset (with squared loss), which is MNIST with only the two classes corresponding to digits $4$ and $7$, with label $-1$ for digit $4$ and $+1$ for digit $7$.

\textbf{Optimisers:} We used stochastic gradient descent (SGD) and \emph{Adam} [\citenum{adam}] as optimisers. In the case of SGD, we tried constant step-sizes in the set $\{0.1,0.01,0.001\}$ and chose the best. In the case of Adam the we used a constant step size of $3e^{-4}$. In both cases, we used batch size to be $32$.  

\textbf{Architectures:} For MNIST we used fully connected (FC) architectures with $(w=128,d=5)$. For CIFAR-10, we used a \emph{`vanilla' convolutional} architecture namely VCONV and a \emph{convolutional} architecture with \emph{global-average-pooling} (GAP) namely GCONV. GCONV had no pooling, residual connections, dropout or batch-normalisation, and is given as follows: input layer is $(32, 32, 3)$, followed by $4$ convolution layers, each with a stride of $(3, 3)$ and channels $64$, $64$, $128$, $128$ respectively. The convolutional layers are followed by GAP layer, and a FC layer with $256$ units, and a soft-max layer to produce the final predictions. VCONV is same as GCONV without the GAP layer.

 \begin{table}[t]
\resizebox{\columnwidth}{!}{
\begin{tabular}{|c|c|c|c|c|c|c|c|}\hline
Arch			&Optimiser	&Dataset		&FRNPF (II) 			&FRNPF (DI)			&DLNPF $(\beta=4)$					&FLNPF						&ReLU\\\hline
FC			&SGD		&MNIST 		&$95.85\pm0.10$		&$95.85\pm0.17$		&$97.86\pm0.11$			&$97.10\pm0.09$				&$97.85\pm0.09$\\\hline
FC			&Adam		&MNIST 		&$96.02\pm0.13$		&$96.09\pm0.12$		&$\mathbf{98.22\pm0.05}$	&$\mathbf{97.82\pm0.02}$		&$\mathbf{98.14\pm0.07}$\\\hline
VCONV		&SGD		&CIFAR-$10$	&$58.92\pm0.62$		&$58.83\pm0.27$ 		&$63.21\pm0.07$			&$63.06\pm0.73$				&$67.02\pm0.43$\\\hline
VCONV		&Adam		&CIFAR-$10$	&$64.86\pm1.18$		&$64.68\pm0.84$		&$\mathbf{69.45\pm0.76}$	&$\mathbf{71.40\pm0.47}$			&$\mathbf{72.43\pm0.54}$\\\hline
GCONV	&SGD		&CIFAR-$10$	&$67.36\pm0.56$		&$66.86\pm0.44$		&$\mathbf{74.57\pm0.43}$			&$\mathbf{78.52\pm0.39}$				&$\mathbf{78.90\pm0.37}$\\\hline
GCONV	&Adam		&CIFAR-$10$	&$67.09\pm0.58$		&$67.08\pm0.27$		&$\mathbf{77.12\pm0.19}$	&$\mathbf{79.68\pm0.32}$		&$\mathbf{80.32\pm0.35}$\\\hline
\end{tabular}
}
\caption{Shows the test accuracy of different NPFs learning settings. Each model is trained close to $100\%$. In each run, the best test accuracy is taken and the table presents values averaged over $5$ runs.}
\label{tb:npfs}
\end{table}
\normalsize
\subsection{Result Discussion}
The results are tabulated in \Cref{tb:npfs}. In what follows, we discuss the key observations.

$1.$ \textbf{Decoupling:} There is no significant performance difference between FRNPF-II and FRNPF-DI regimes, i.e., the statistical independence of $\Tv_0$ and $\Tf_0$ did not affect performance. Also, DLNPF performed better than FRNPF, which shows  that the NPFs can also be learnt in a decoupled manner.

$2.$ \textbf{Performance gain of CNN over CNTK} can be explained by NPF learning. For this purpose, we look at the performance of GCONV models trained with \emph{Adam} on CIFAR-10 (last row of \Cref{tb:npfs}). Consider models grouped as $S_1=\{$FRNPF-DI,FRNPF-II$\}$, $S_2=\{$CNTK$\}$ that have no NPF learning versus models grouped as $S_3=\{$ FLNPF, ReLU$\}$ that have either NPF learning during training or a fixed learnt NPF. The group $S_3$ performs better than $S_1$ and $S_2$. Note that, both $S_1$ and $S_3$ are finite width networks, yet, performance of $S_1$ is worse than CNTK, and the performance of $S_3$ is better than CNTK. Thus finite width alone does not explain the performance gain of CNNs over CNTK. Further, all models in group $S_3$ are finite width and also have NPF learning. Thus, finite width together with NPF learning explains the performance gain of CNN over CNTK. 

$3.$ \textbf{Standard features vs NPFs:} The standard view is that the outputs of the intermediate/hidden layers learn lower to higher level features (as depth proceeds) and the final layer learns a linear model using the hidden features given by the penultimate layer outputs. This view of feature learning holds true for all the models in $S_1$ and $S_3$. However, only NPF learning clearly \emph{discriminates} between the different regimes $S_1, S_2$ and $S_3$. Thus, NPF learning is indeed a unique aspect in deep learning. Further, in the FLNPF regime, using the learnt NPFs and training the NPV from scratch, we can recover the test accuracy. Thus almost all useful information is stored in the gates, a novel observation which underscores the need to further investigate the role of the gates.

$4.$ \textbf{Continuous NPF Learning:} The performance gap between FRNPF and ReLU is continuous. We trained a standard ReLU-CNN with GCONV architecture (parameterised by $\bar{\Theta}$) for $60$ epochs. We sampled $\bar{\Theta}_t$ at various \emph{stages} of training, where stage $i$ corresponds to $\bar{\Theta}_{10\times i}, i=1,\ldots,6$. For these $6$ stages, we setup $6$ different FLNPFs, i.e., FLNPF-$1$ to FLNPF-$6$. We observe that the performance of FLNPF-$1$ to FLNPF-$6$ increases monotonically, i.e., FLNPF-$1$ performs ($72\%$) better than FRNPF ($67.09\%$),  and FLNPF-$6$ performs as well as ReLU (see left most plot in \Cref{fig:dynamics}). The performance of CNTK of \cite{arora2019exact} is $77.43\%$. Thus, through its various stages, FLNPF starts from below $77.43\%$ and surpasses to reach $79.43\%$, which implies performance gain of CNN is due to learning of NPFs.

\begin{figure}[h]
\centering
\includegraphics[scale=0.23]{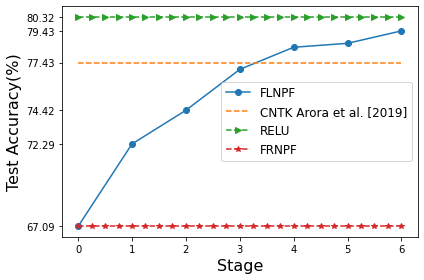}
\includegraphics[scale=0.23]{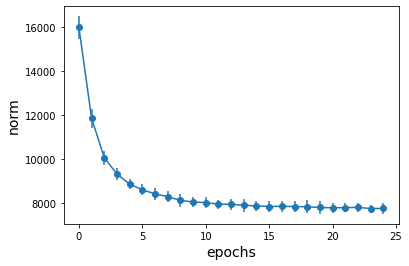}
\includegraphics[scale=0.23]{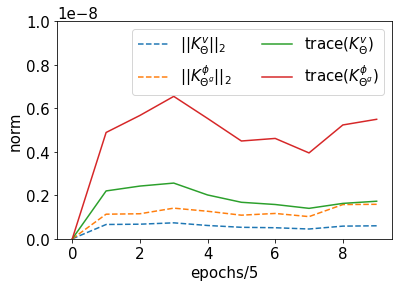}
\caption{Dynamics of NPF Learning. }
\label{fig:dynamics}
\end{figure}

$5.$ \textbf{Dynamics of active sub-networks during training:} We trained a FC network ($w=100$, $d=5$) on the ``Binary''-MNIST dataset. Let $\widehat{H}_{\Theta_t}=\frac{1}{trace(H_{\Theta_t})}H_{\Theta_t}$ be the normalised NPK matrix. For a subset size, $n'=200$ ($100$ examples per class) we plot $\nu_t=y^\top (\widehat{H}_{\Theta_t})^{-1} y$, (where $y\in\{-1,1\}^{200}$ is the labelling function), and observe that $\nu_t$ reduces as training proceeds (see middle plot in \Cref{fig:dynamics}). Note that, $\nu_t=\sum_{i=1}^{n'}(u_{i,t}^\top y)^2 (\hat{\rho}_{i,t})^{-1}$, where $u_{i,t}\in \R^{n'}$ are the orthonormal eigenvectors of $\widehat{H}_{\Theta_t}$ and $\hat{\rho}_{i,t},i\in[n']$ are the corresponding eigenvalues. Since $H_{\Theta_t}=\Sigma\odot \Lambda_{\Theta_t}$, we can infer that $\Lambda_{\Theta_t}$ is learnt during training.

$6.$ \textbf{How are NPFs learnt?} In order to understand this, in the case of DNNs with ReLU, for the purpose of analysis,  we can replace the hard-ReLU gate by the soft-ReLU gate. Now, the gradient is given by $\partial_{\theta}\hat{y}_{\Theta}(x)=\ip{\partial_{\theta} \phi_{x,\Theta},v_{\Theta}}+\ip{\phi_{x,\Theta},\partial_{\theta} v_{\Theta}}$, where the two terms on the right can be interpreted as NPF and NPV gradients. Using the soft-ReLU ensures $\psif\neq 0$ (note that $\psif=0$ for hard-ReLU due to its $0/1$ state). We can obtain further insights in the DLNPF regime, wherein, the NTK is given by $K_{\Tdgn}=K^{\text{v}}_{\Tdgn}+K^{\text{f}}_{\Tdgn}$. For MNIST, we compared $K^{\text{v}}_{\Tdgn}$ and $K^{\text{f}}_{\Tdgn}$ (calculated on $100$ examples in total with $10$ examples per each of the $10$ classes) using their trace and Frobenius norms, and we observe that $K^{\text{v}}_{\Tdgn}$ and $K^{\text{f}}_{\Tdgn}$ are in the same scale (see right plot in \Cref{fig:dynamics}), which perhaps shows that both $K^{\text{f}}_{\Tdgn}$ and $K^{\text{v}}_{\Tdgn}$ are equally important for obtaining good generalisation performance. Studying $K^{\text{f}}_{\Tdgn}$ responsible for NPF learning is an interesting future research direction.

\section{Related Work}
\setcitestyle{authoryear}
\textbf{Gated Linear Unit (GaLU)} networks with a single hidden layer was theoretically analysed by \cite{sss,fiat}. In contrast, our NPFs/NPV formulation enabled us to analyse DGNs of any depth $d$. The fixed random filter setting of \cite{fiat} is equivalent to (in our setting) the FRNPF regime of DGNs with a single hidden layer. They test the hypothesis (on MNIST and Fashion-MNIST) that the effectiveness of ReLU networks is mainly due to the training of the linear part (i.e., weights) and not the gates. They also show that a ReLU network marginally outperforms a GaLU network (both networks have a single hidden layer), and propose a non-gradient based algorithm (which has a separate loop to randomly search and replace the gates of the GaLU network) to close the margin between the GaLU and ReLU networks. We observe a similar trend in that, the FRNPFs (here only NPV is learnt) do perform well in the experiments with a test accuracy of around $68\%$ on CIFAR-10. However, it is also the case that models with NPF learning perform significantly better (by more than $10\%$) than FRNPF, which underscores the importance of the gates. Further, we used soft-ReLU gates in the DLNPF regime, and showed that using standard optimisers (based on gradient descent), we can learn the NPFs and NPV in two separate networks. In addition, capturing the role of the active sub-networks via the NPK is another significant progress over \cite{sss,fiat}.

\textbf{Prior NTK works:} \cite{ntk} showed the NTK to be the central quantity in the study of generalisation properties of infinite width DNNs. \cite{jacot2019freeze} identify two regimes that occur at initialisation in fully connected infinite width DNNs namely i) \emph{freeze:} here, the (scaled) NTK converges to a constant and hence leads to slow training,  and  ii) \emph{chaos:} here, the NTK converges to Kronecker delta and hence hurts generalisation. \cite{jacot2019freeze} also suggest that for good generalisation it is important to operate the DNNs at the edge of the freeze and the chaos regimes. The works of \cite{arora2019exact,cao2019generalization} are closely related to our work and have been discussed in the introduction.
\cite{dudnn} use the NTK to show that over-parameterised DNNs trained by gradient descent achieve zero training error. \cite{dudln,shamir,ganguli} studied deep linear networks (DLNs). Since DLNs are special cases of DGNs, \Cref{th:main} of our paper also provides an expression for the NTK at initialisation of deep linear networks. To see this, in the case of DLNs, all the gates are always $1$ and $\Lambda_{\Theta}$ is a matrix whose entries will be $w^{(d-1)}$. 

\textbf{Other works:} Empirical analysis of the role of the gates was done by \cite{srivastava2014understanding}, where the active sub-networks are called as \emph{locally competitive} networks. Here, a `sub-mask' encodes the $0/1$ state of all the gates. A `t-SNE' visualisation of the sub-masks showed that ``subnetworks active for examples of the same class are much more similar to each other compared to the ones activated for the examples of different classes''. \cite{balestriero2018spline}  connect $\max$-affine linearity and DNN with ReLU activations. \cite{neyshabur2015path} proposed a novel \emph{path-norm} based gradient descent. 

\section{Conclusion}
In this paper, we developed a novel neural path framework to capture the role of gates in deep learning. We showed that the neural path features are learnt during training and such learning improves generalisation. In our experiments, we observed that almost all information of a trained DNN is stored in the neural path features. We conclude by saying that \emph{understanding deep learning requires understanding neural path feature learning}. 

\section{Broader Impact}
Deep neural networks are still widely regarded as blackboxes. The standard and accepted view on the inner workings of  deep neural networks is the `layer-by-layer' viewpoint:  as the input progresses through the hidden layers, features at different levels of abstractions are learnt. This paper deviates from the standard `layer-by-layer' viewpoint, in that, it breaks down the deep neural network blackbox into its constituent paths: different set of paths get fired for different inputs, and the output is the summation of the contribution from individual paths. This makes the inner workings of deep neural networks interpretable, i.e., each input is remembered in terms of the active sub-network of the paths that get `fired' for that input, and learning via gradient descent amounts to `rewiring' of the paths. The paper also analytically connects this sub-network and path based view to the recent kernel based interpretation of deep neural networks, and furthers the understanding of feature learning in deep neural networks. We believe that these better insights into the working of DNNs  can potentially lead to foundational algorithmic development in the future.

\section*{Acknowledgements}
We thank Harish Guruprasad Ramaswamy, Arun Rajkumar, Prabuchandran K J, Braghadeesh Lakshminarayanan and the anonymous reviewers for their valuable comments. We would also like to thank Indian Institute of Technology Palakkad for the `Seed Grant', and Science and Engineering Research Board (SERB), Department of Science and Technology, Government of India for the `Startup Research Grant' (SRG/2019/001856). 

\setcitestyle{numbers}
\bibliographystyle{plainnat}
\bibliography{refs}
\input{appendix}

\end{document}

%% file: pack.tex
\usepackage{multicol}
\usepackage{multirow}
\usepackage{dsfont}
\usepackage{tabularx}
\usepackage[textsize=tiny]{todonotes}
\usepackage{amsmath}
\usepackage{mathtools}
\usepackage{graphicx}
\usepackage{times}
\usepackage{helvet}
\usepackage{courier}
\usepackage{paralist}
\usepackage{latexsym}
\usepackage{url}
\usepackage[all]{xy}
\usepackage{amsmath}
\usepackage{amssymb}
\usepackage{amsthm}
\usepackage{nccmath} 
\usepackage{comment}
\usepackage{paralist}
\usepackage{xcolor}
\usepackage{graphicx}
\usepackage{pifont}
\usepackage{savesym}
\savesymbol{AND}
\usepackage{xspace}
\usepackage{tikz}
\usepackage{pgfplots}
\usepackage{pgf}
\usepackage{algorithm}
\usepackage{algorithmic}
\usepackage{xspace}
\usepackage{comment}
\usepackage{placeins}

\newcommand{\nn}{\nonumber}

\newcommand{\R}{\Re} 

\newcommand{\ra}{\rightarrow}

\newcommand{\E}[1]{\mathbb{E}\left[#1\right]}

\newcommand{\B}{\mathcal{B}}

\newcommand{\V}{\mathcal{V}}
\newcommand{\I}{\mathcal{I}}

\newcommand{\psiv}{\psi^{\text{v}}}
\newcommand{\psif}{\psi^{\text{f}}}

\newcommand{\kf}{K^{\text{f}}_{\Theta}}

\newcommand{\Tv}{{\Theta}^{\text{v}}}
\newcommand{\tv}{{\theta^{\text{v}}}}
\newcommand{\Tf}{{\Theta}^{\text{f}}}

\newcommand{\Tdgn}{\Theta^{\text{DGN}}}

\newcommand{\din}{d_{\text{in}}}
\newcommand{\dnet}{d_{\text{net}}}
\newcommand{\dvnet}{d^{\text{v}}_{\text{net}}}
\newcommand{\dfnet}{d^{\text{f}}_{\text{net}}}
\newcommand{\ta}{\theta_\text{a}}
\newcommand{\tb}{\theta_\text{b}}

\newcommand{\pta}{\partial_{\ta}}
\newcommand{\ptb}{\partial_{\tb}}

\newcommand{\norm}[1]{\|#1\|}

\newcommand{\eqdef}{\doteq}
\renewcommand{\eqdef}{\stackrel{def}{=}}

\renewcommand{\epsilon}{\varepsilon}

\newtheorem{theorem}{Theorem}[section]
\newtheorem{lemma}{Lemma}[section]

\newtheorem{proposition}{Proposition}[section]
\newtheorem{corollary}[theorem]{Corollary}

\newtheorem{definition}{Definition}[section]

\newcommand{\ip}[1]{\langle #1\rangle}

\def\Re{\mathbb{R}}
\def\Z{\mathbb{Z}}

\newcounter{subequation}[equation]

\def\mathdisplay#1{%
  \ifmmode \@badmath
  \else
    $\def\@currenvir{#1}%
    \let\dspbrk@context\z@
    \let\tag\tag@in@display \SK@equationtrue 
    \global\let\df@label\@empty \global\let\df@tag\@empty
    \global\tag@false
    \let\mathdisplay@push\mathdisplay@@push
    \let\mathdisplay@pop\mathdisplay@@pop
    \if@fleqn
      \edef\restore@hfuzz{\hfuzz\the\hfuzz\relax}%
      \hfuzz\maxdimen
      \setbox\z@\hbox to\displaywidth\bgroup
        \let\split@warning\relax \restore@hfuzz
        \everymath\@emptytoks \m@th $\displaystyle
    \fi
}

\newcounter{algostep}

\newcounter{acalgorithm}

%% file: appendix.tex
\newpage
\begin{center}
{\Large{\textbf{Appendix}}}
\end{center}

\begin{appendix}
\section{Expression for $K^{(d)}$}\label{sec:kd}
The $K^{(d)}$ matrix is computed by the recursion in \eqref{eq:ntkold}.
\begin{align}\label{eq:ntkold}
&\tilde{K}^{(1)}(s,s')=\Sigma^{(1)}(s,s')=\Sigma(s,s'), M^{(l)}_{ss'}=\left[\begin{matrix}\Sigma^{(l)}(s,s) & \Sigma^{(l)}(s,s')\\ \Sigma^{(l)}(s',s) & \Sigma^{(l)}(s',s')\end{matrix}\right]\in \R^2,\nn\\
&\Sigma^{(l+1)}(s,s')= 2\cdot\mathbb{E}_{(q,q')\sim N(0,M_{ss'}^{(l)})} \left[\chi(q)\chi(q')\right], \hat{\Sigma}^{(l+1)}(s,s')= 2\cdot\mathbb{E}_{(q,q')\sim N(0,M_{ss'}^{(l)})}\left[\partial\chi(q)\partial{\chi}(q')\right],\nn\\
&\tilde{K}^{(l+1)}=\tilde{K}^{(l)}\odot \hat{\Sigma}^{(l+1)}+\Sigma^{(l+1)}, K^{(d)}=\left(\tilde{K}^{(d)}+\Sigma^{(d)}\right)/2
\end{align}
where $s,s'\in[n]$ are two input examples in the dataset, $\Sigma$ is the data Gram matrix, $\partial{\chi}$ stands for the derivative of the activation function with respect to the pre-activation input, $N(0,M)$ stands for the mean-zero Gaussian distribution with co-variance matrix $M$.

\section{Proofs of technical results}
\textbf{Proof of \Cref{prop:basic}}
\begin{proof}
We know that $e_t=(e_t(s),s\in[n])\in\R^n$, and $e_t(s)=\hat{y}_{\Theta_t}(x_s)-y(s)$. Now
\begin{align} 
L_{\Theta_t}&=\frac{1}2\sum_{s'=1}^n(\hat{y}_{\Theta_t}-y)^2\nn\\
&=\frac{1}2\sum_{s'=1}^n e_t^2\nn\\
\nabla_{\Theta} L_{\Theta_t}&= \sum_{s'=1}^n\nabla_{\Theta} \hat{y}_{\Theta_t}(x_{s'})e_t(s')\nn\\
\label{eq:above1} \nabla_{\Theta} L_{\Theta_t}&= \sum_{s'=1}^n \psi_{x_{s'},\Theta_t}e_t(s')
\end{align}
For gradient descent, $\dot{\Theta}_t=-\nabla_{\Theta} L_{\Theta_t}$, from \eqref{eq:above1} it follows that 
\begin{align}
\dot{\Theta}_t=-\sum_{s'=1}^n \psi_{x_{s'},\Theta_t}e_t(s')
\end{align}
Now $\dot{e}_t=\dot{\hat{y}}_{\Theta_t}$, and expanding $\dot{\hat{y}}_{\Theta_t}(x_s)$ for some $s\in[n]$, we have:
\begin{align}
\dot{\hat{y}}_{\Theta_t}(x_s)&=\frac{d \hat{y}_{\Theta_t}(x_s)}{d t}\nn\\
&=\sum_{\theta\in\Theta}\frac{d \hat{y}_{\Theta_t}(x_s)}{d \theta}\frac{d \theta_t}{dt},\,\text{by expressing this summation as a dot product we obtain} \nn\\
\dot{\hat{y}}_{\Theta_t}(x_s)&=\ip{\psi_{x_s,\Theta_t},\dot{\Theta}_t}
\end{align}
We now use that fact that $\Theta_t$ is updated by gradient descent
\begin{align}
\dot{\hat{y}}_{\Theta_t}(x_s)&=-\ip{\psi_{x_s,\Theta_t},\sum_{s'=1}^n \psi_{x_{s'},\Theta_t}e_t(s')}\nn\\
&=-\sum_{s'=1}^n K_{\Theta_t}(s,s')e_t(s')
\end{align}
The proof is complete by recalling that $\hat{y}_{\Theta_t}=(\hat{y}_{\Theta_t}(x_s),s\in[n])$, and $\dot{e}_t=\dot{\hat{y}}_{\Theta_t}$.
\end{proof}

\textbf{Proof of \Cref{prop:zero}}
\begin{proof}
Let $x\in\R^{\din}$ be the input to the DNN and $\hat{y}_{\Theta}(x)$ be its output. The output can be written in terms of the final hidden layer output as:
\begin{align}
\hat{y}_{\Theta}(x)&=\sum_{j_{d-1}=1}^w\Theta(1, j_{d-1},d) \cdot z_{x,\Theta}(j_{d-1},d-1)\nn\\
\label{lastlayer}&=\sum_{j_{d-1}=1}^w\Theta(1, j_{d-1},d) \cdot G_{x\Theta}(j_{d-1},d-1)\cdot q_{x,\Theta}(j_{d-1},d-1)
\end{align}
Now $q_{x,\Theta}(j_{d-1},d-1)$ for a fixed $j_{d-1}$ can again be expanded as
\begin{align}
q_{x,\Theta}(j_{d-1},d-1)&= \sum_{j_{d-2}=1}^w \Theta(j_{d-1},j_{d-2},d-1) \cdot z_{x,\Theta}(j_{d-2},d-2)\nn\\
\label{onebefore}&=\sum_{j_{d-2}=1}^w \Theta(j_{d-1},j_{d-2},d-1) \cdot G_{x,\Theta}(j_{d-2},d-2)\cdot q_{x,\Theta}(j_{d-2},d-2)
\end{align}
Now plugging in \eqref{onebefore} in the expression in \eqref{lastlayer}, we have
\begin{align}
\hat{y}_{\Theta}(x)&=\sum_{j_{d-1}=1}^w\Theta(1, j_{d-1},d)\cdot G_{x\Theta}(j_{d-1},d-1)\Bigg(\sum_{j_{d-2}=1}^w \Theta(j_{d-1},j_{d-2},d-1)\nn\\ &\hspace{15pt}\cdot G_{x,\Theta}(j_{d-2},d-2)\cdot q_{x,\Theta}(j_{d-2},d-2)\Bigg)\nn\\
&=\sum_{j_{d-1}, j_{d-2}\in[w]} G_{x,\Theta}(j_{d-1},d-1)\cdot G_{x,\Theta}(j_{d-2},d-2)\cdot\Theta(1, j_{d-1},d)\nn\\&\cdot \Theta(j_{d-1},j_{d-2},d-1)\cdot q_{x,\Theta}(j_{d-2},d-2)\nn\\
\end{align}
By expanding $q$'s for all the previous layers till the input layer we have
\begin{align}
\hat{y}_{\Theta}(x)=\sum_{j_{d}=1, j_{d-1},\ldots,j_{1}\in[w], j\in[\din]} x(j) \Pi_{l=1}^{d-1}G_{x,\Theta}(j_{l},l) \Pi_{l=1}^{d}\Theta(j_l, j_{l-1}, l) \nn
\end{align}
\end{proof}

\textbf{Proof of \Cref{lm:npk}}
\begin{proof}
\begin{align}
\ip{\phi_{x_s,\Theta},\phi_{x_{s'},\Theta}}&=\sum_{p\in[P]}x_s(\I_0(p))x_{s'}(\I_0(p))A_{\Theta}(x_s,p)A_{\Theta}(x_{s'},p)\nn\\
&=\sum_{i=1}^{\din}x_s(i)x_{s'}(i)\Lambda_{\Theta}(s,s')\nn\\
&=\ip{x_s,x_{s'}}\cdot\Lambda_{\Theta}(s,s')
\end{align}
\end{proof}

    \textbf{Proof of \Cref{prop:ntknew}}
\begin{proof}
Let $\Psi_{\Theta}=(\psi_{x_s,\Theta},s\in[n])\in\R^{\dnet\times n}$ be the NTF matrix, then the NTK matrix is given by $K_{\Theta_t}=\Psi^\top_{\Theta_t}\Psi_{\Theta_t}$. Note that, $\hat{y}_{\Theta}(x_s)=\ip{\phi_{x_s,\Theta},v_{\Theta}}=\ip{v_{\Theta},\phi_{x_s,\Theta}}=v^\top_{\Theta}\phi_{x_s,\Theta}$. Now $\psi_{x_{s},\Theta}=\nabla_{\Theta} v_{\Theta}\phi_{x_s,\Theta}$, and hence $\Psi=\nabla_{\Theta} v_{\Theta}\Phi_{\Theta}$. Hence, $K_{\Theta_t}=\Psi^\top_{\Theta_t}\Psi_{\Theta_t}=\Phi^\top_{\Theta}(\nabla_{\Theta} v_{\Theta})^\top (\nabla_{\Theta} v_{\Theta})\Phi_{\Theta}=\Phi^\top_{\Theta}\V_{\Theta}\Phi_{\Theta}$.
\end{proof}

\textbf{Proof of \Cref{prop:dnnhard}}

\begin{proof}
Follows in a similar manner as the proof of \Cref{prop:basic}.
\end{proof}

\textbf{Proof of {\Cref{prop:condition}}}
\begin{proof}
$\rho_{\min}(K_{\Theta})=\underset{\norm{x}_2=1}{\underset{x\in \R^n}\min}x^\top K_{\Theta} x$. Let $x'\in\R^n$ such that $\norm{x'}_2=1$ and $\rho_{\min}(H_{\Theta})={x'}^\top H_{\Theta} x'$. Now, $\rho_{\min}(K_{\Theta})\leq {x'}^\top K_{\Theta} x'$. 
Let $y'=\Phi x'$, then we have, $\rho_{\min}(K_{\Theta})\leq {y'}^\top \V_{\Theta}y'$. Hence $\rho_{\min}(K_{\Theta})\leq \norm{y'}^2_2 \rho_{\max}(\V_{\Theta})$. Proof is complete by noting that $\norm{y'}^2_2={x'}^\top \Phi^\top_{\Theta}\Phi_{\Theta}x'= \rho_{\min}(H_{\Theta})$.
\end{proof}

\textbf{Proof of \Cref{prop:dgn}}

\begin{proof}
Follows in a similar manner as proof of \Cref{prop:basic}.
\end{proof}

\subsection{Proof of \Cref{th:main}}
\subsubsection{Calculation of $\E{K^\text{v}_{\Tdgn_0}}$}

\begin{proposition}
Let $\tv\in\Tv$ be a weight in layer $l_{\tv}$, and let $p$ be a path that passes through $\tv$. Then 
\begin{align}
\partial_{\tv} v_{\Tv}(p) =& \Pi_{l=1,l\neq l_{\tv}}^{d} \Theta(\I_{l}(p),\I_{l-1}(p),l )
\end{align}
\end{proposition}
\begin{proof}
Proof follows by noting that $v_{\Tv}(p)=\Pi_{l=1}^{d}\Theta(\I_l(p),\I_{l-1}(p),l)$.
\end{proof}

\begin{lemma}\label{lm:dot}
Let $\varphi_{p,\Theta}$ be as in \Cref{def:npvgrad}, under the assumption in ~\Cref{th:main}, for paths $p_1,p_2\in [P], p_1\neq p_2$, at initialisation we have (i) $\E{\ip{\varphi_{p_1,\Tv_0}, \varphi_{p_2,\Tv_0}}}= 0$, (ii) ${\ip{\varphi_{p_1,\Tv_0}, \varphi_{p_1,\Tv_0}}}= d\cdot \sigma^{2(d-1)}$.
\end{lemma}

\begin{proof}
\begin{align*}
\ip{\varphi_{p_1,\Tv_0}, \varphi_{p_2,\Tv_0}}= \sum_{\tv\in \Tv} \partial_{\tv}v_{\Tv_0}(p_1) \partial_{\tv}v_{\Tv_0}(p_2)
\end{align*}
Let $\tv\in\Tv$ be an arbitrary weight. If either $p_1$ or $p_2$ does not pass through $\tv$, then it follows that $\partial_{\tv} v_{\Tv_0}(p_1) \partial_{\tv} v_{\Tv_0}(p_2)=0$. Let us consider the the case when $p_1,p_2$ pass through $\tv$ and without of loss of generality let $\tv$ belong to layer $l_{\tv}\in[d]$. 
  we have
\begin{align*}
&\E{\partial_{\tv}v_{\Tv_0}(p_1)\partial_{\tv}v_{\Tv_0}(p_2)}\\
&=\E{\underset{l\neq l_{\tv}}{\underset{l=1}{\overset{d}{\Pi}}} \Bigg(\Tv_0(\I_{l} (p_1),  \I_{l-1}(p_1),l) \Tv_0(\I_{l}(p_2),\I_{l-1} (p_2),l) \Bigg)}\\
&=\underset{l\neq l_{\tv}}{\underset{l=1}{\overset{d}{\Pi}}}\E{\Tv_0(\I_{l}(p_1),\I_{l-1}(p_1),l)\Tv_0(\I_{l}(p_2),\I_{l-1}(p_2),l)}
\end{align*}
where the $\E{\cdot}$ moved inside the product because at initialisation the weights (of different layers) are independent of each other.
Since $p_1\neq p_2$, there exist a layer $\tilde{l}\in[d],\tilde{l}\neq l_{\tv}$ such that they do not pass through the same weight in layer $\tilde{l}$, i.e., $\Tv_0(\I_{\tilde{l}}(p_1),\I_{\tilde{l}-1}(p_1),\tilde{l},)$ and $\Tv_0(\I_{\tilde{l}}(p_2),\I_{\tilde{l}-1}(p_2),\tilde{l})$ are distinct weights. Using this fact,  we have 
\begin{align*}
&\E{\partial_{\tv}v_{\Tv_0}(p_1)\partial_{\tv}v_{\Tv_0}(p_2)}\\
=&\Bigg(\underset{l\neq l_{\tv},\tilde{l}}{\underset{l=1}{\overset{d}{\Pi}}}\E{\Tv_0(\I_l(p_1), \I_{l-1}(p_1),l)\Tv_0(\I_{l}(p_2),\I_{l-1}(p_2),l)}\Bigg)\\
&\cdot\Bigg(\E{\Tv_0(\I_{\tilde{l}}(p_1),\I_{\tilde{l}-1} (p_1),\tilde{l})}\E{\Tv_0(\I_{\tilde{l}}(p_2), \I_{\tilde{l}-1 }(p_2),\tilde{l})}\Bigg)\\
=&0
\end{align*}

The proof of (ii) is complete by noting that a given path $p_1$ pass through only `$d$' weights, and hence $\sum_{\tv\in\Tv} \partial_{\tv}v_{\Tv_0}(p_1) \partial_{\tv}v_{\Tv_0}(p_1)$ has `$d$' non-zero terms, and the fact that at initialisation we have 
\begin{align*}
&\partial_{\tv}v_{\Tv_0}(p_1) \partial_{\tv}v_{\Tv_0}(p_1) \\
&=\underset{l\neq l_{\tv}}{\underset{l=1}{\overset{d}{\Pi}}} [\Tv_0(\I_{l}(p),\I_{l-1}(p),l)]^2\\
&=\sigma^{2(d-1)}
\end{align*}
\end{proof}

\begin{theorem}\label{th:exp}
 $\E{K^{\text{v}}_{\Tdgn_0}}=d\cdot\sigma^{2(d-1)} \cdot H_{\text{FNPF}}$. 
\end{theorem}
\begin{proof}
Let $\Phi_{\text{FNPF}}=\Phi_{\Tf_0}=\left(\phi_{x_s,\Tf_0}, s\in[n]\right)\in\R^{P\times n}$ be the NPF matrix. 
\begin{align*}
\E{K^{\text{v}}_{\Tdgn_0}}&=\E{\Phi^\top_{\text{FNPF}} \V_{\Tv_0} \Phi_{\text{FNPF}}}\\
&=\E{\Phi^\top_{\text{FNPF}} (\nabla_{\Tv}v_{\Tv_0})^\top (\nabla_{\Tv}v_{\Tv_0}) \Phi_{\text{FNPF}}}\\
&=\Phi^\top_{\text{FNPF}}\left( \E{(\nabla_{\Tv}v_{\Tv_0})^\top (\nabla_{\Tv}v_{\Tv_0})}\right)\Phi_{\text{FNPF}}\\
&\stackrel{(a)}=d\cdot\sigma^{2(d-1)} \cdot\left(\Phi^\top_{\text{FNPF}}\Phi_{\text{FNPF}}\right)\\
&=d\cdot\sigma^{2(d-1)} \cdot H_{\text{FNPF}}
\end{align*}
Here, $(a)$ follows from  \Cref{lm:dot}, i.e., $\E{(\nabla_{\Tv}v_{\Tv_0})^\top (\nabla_{\Tv}v_{\Tv_0})}= d\cdot\sigma^{2(d-1)}\cdot I_{P\times P}$, where $I_{P\times P}$ is a ${P\times P}$ identity matrix.
\end{proof}

\subsubsection{Calculation of $Var\left[K^\text{v}_{\Tdgn_0}\right]$}

\input{varproof}

\textbf{Proof of\Cref{th:main}}
\begin{proof} Follows from \Cref{th:exp} and \Cref{th:var}.
\end{proof}

\section{Applying \Cref{th:main} In Finite Width Case}\label{sec:finite}
In this section, we describe the technical step in applying \Cref{th:main} which requires $w\ra\infty$ to measure the information in the gates of a DNN  with finite width.  Since we are training only the value network in the FPNP mode of the DGN, it is possible to let the width of the value network alone go to $\infty$, while keeping the width of the feature network (which stores the fixed NPFs) finite. This is easily achieved by multiplying the width by a positive integer $m\in\Z_{+}$, and \emph{padding} the gates `$m$' times.
\begin{definition}
Define DGN${}^{(m)}$ to be the DGN whose feature network is of width $w$ and depth $d$, and whose value network  is a fully connected network of width $mw$ and depth $d$. The $mw(d-1)$ gating values are obtained by `padding' the $w(d-1)$gating values of the width `$w$', depth `$d$' feature network `$m$' times (see \Cref{fig:dgnpad}, \Cref{tb:dgnpad}). 
\end{definition}
\FloatBarrier
\begin{table}[h]
\centering
\begin{tabular}{|  l | l |}\hline
 Feature Network (NPF)& Value Network (NPV)\\
 $z^{\text{f}}_{x}(0)=x$ &$z^{\text{v}}_{x}(0)=x$ \\
$q^{\text{f}}_{x}(i,l)=\sum_{j} \Tf(i,j,l)\cdot z_{x}(j,l-1)$ & $q^{\text{v}}_{x}(i,l)=\sum_{j} \Tv(i,j,l)\cdot z^{\text{v}}_{x}(j,l-1)$\\
$z^{\text{f}}_{x}(i,l)= q^{\text{f}}_{x}(i,l)\cdot\mathbbm{1}_{\{q^{\text{f}}_{x}(i,l)>0\}}$& $z^{\text{v}}_{x}(i,l)= q^{\text{v}}_{x}(i,l)\cdot G_{x}(i,l)$ \\
 None &$\hat{y}_{{\Tdgn}^{(m)}}(x)= \sum_{j} \Tv(1,j,l)\cdot z^{\text{v}}_{x}(j,d-1)$\\\hline
\multicolumn{2}{|l|}{{Hard ReLU: $G_{x}(i,l)=\mathbbm{1}_{\{q^{\text{f}}_{x}(i,l)>0\}}$ or Soft-ReLU: $G_{x}(i,l)={1}/{\left(1+\exp(-\beta\cdot q^{\text{f}}_{x}(i,l)>0)\right)} $}}\\\hline
\end{tabular}
\caption{Deep Gated Network with padding. Here the gating values are padded, i.e., $ G_{x}(kw+i,l)=G_{x}(i,l),\forall k=0,1,\ldots,m-1, i\in[w]$. }
\label{tb:dgnpad}
\end{table}

\textbf{Remark:}  DGN${}^{(m)}$ has a total of $P^{(m)}=(mw)^{(d-1)}\din$ paths. Thus, the NPF and NPV are quantities in $\R^{P^{(m)}}$. In what follows, we denote the NPF matrix of DGN${}^{(m)}$ by $\Phi^{(m)}_{\Tf_0}\in\R^{P^{(m)}\times n}$, and use $H^{(m)}_{\text{FNPF}}=(\Phi^{(m)}_{\Tf_0})^\top \Phi^{(m)}_{\Tf_0}$. 

Before we proceed to state the version of \Cref{th:main} for DGN${}^{(m)}$, we will look at an equivalent definition for $\Lambda_{\Theta}$ (see \Cref{def:lambda}).
\begin{definition}\label{def:equilambda}
For input examples $s, s'\in[n]$ define 

$1.$ $\tau_{\Theta}(s,s',l)\stackrel{def}=\sum_{i=1}^w G_{x_s,\Theta}(i,l)G_{x_{s'},\Theta}(i,l)$ be the number of activations that are ``on'' for both inputs $s,s'\in[n]$ in layer $l\in[d-1]$.

$2.$ $\Lambda_{\Theta}(s,s')\stackrel{def}=\Pi_{l=1}^{d-1}\tau_{\Theta}(s,s',l)$.
\end{definition}

\FloatBarrier
\begin{figure}[h]
\centering
\includegraphics[scale=0.1]{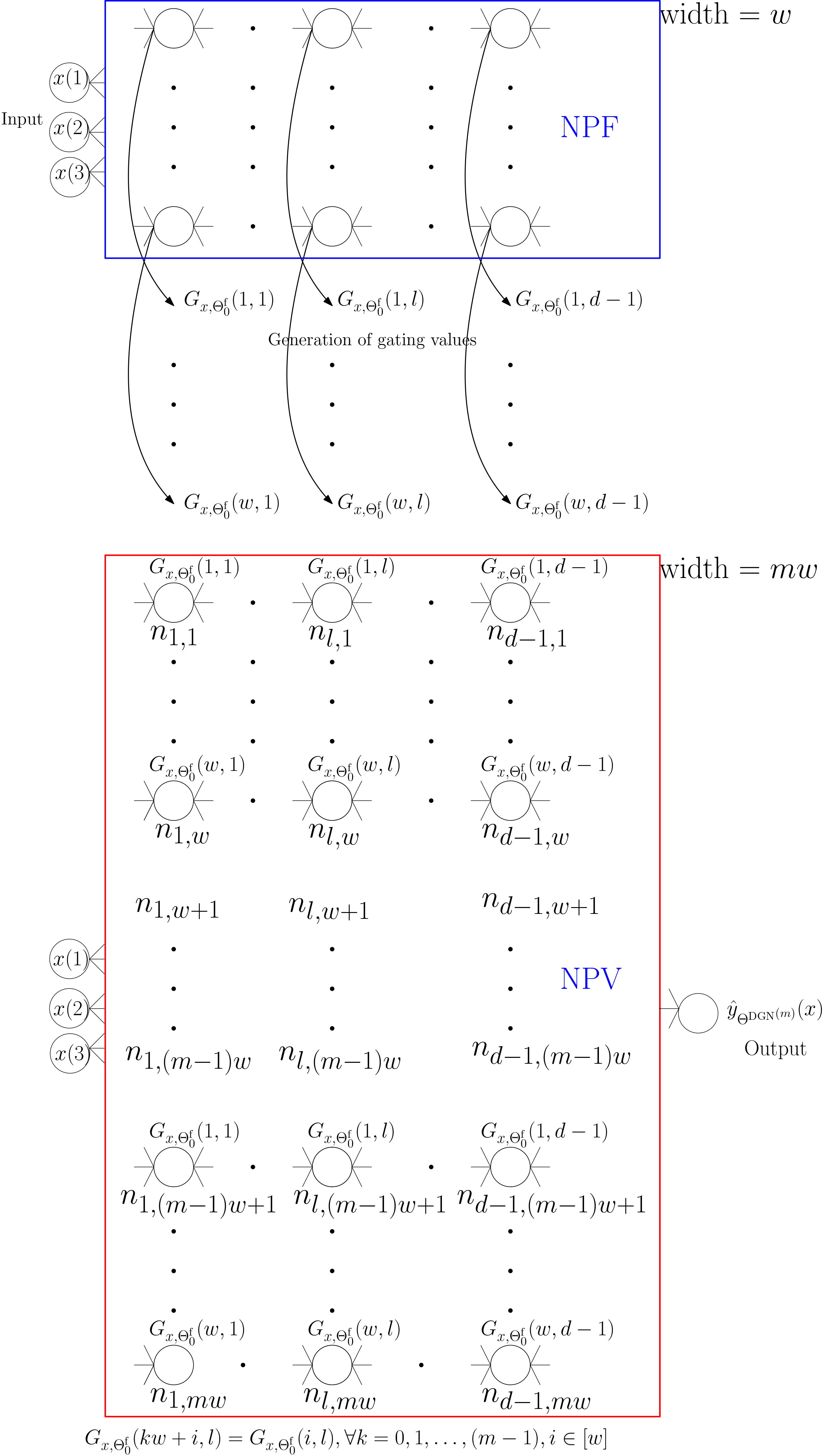}
\caption{DGN${}^{(m)}$ where the value network is of width $mw$ and depth $d$. The gates are derived by padding the gating values obtained from the feature network `$m$' times, i.e., $ G_{x}(kw+i,l)=G_{x}(i,l),\forall k=0,1,\ldots,m-1, i\in[w]$.}
\label{fig:dgnpad}
\end{figure}

\begin{corollary}[Corollary to \Cref{th:main}] Under the same assumptions as in \Cref{th:main} with $\sigma$ replaced by $\sigma_{(m)}=\sigma/\sqrt{m}$, as $m\ra\infty$, \begin{align*}K^{\text{v}}_{\Theta^{\text{DGN}^{(m)}}_0}\ra K^{(d)}_{\text{FNPF}} = d\cdot \sigma_{(m)}^{2(d-1)} H^{(m)}_{\text{FNPF}}= d\cdot \sigma^{2(d-1)} H_{\text{FNPF}}\end{align*}
\end{corollary}
\begin{proof}
Let $\Lambda^{(m)}_{\text{FNPF}}$ and $\tau^{(m)}_{\text{FNPF}}$ be quantities associated with DGN${}^{(m)}$. We know that  $H^{(m)}_{\text{FNFP}}=\Sigma\odot\Lambda^{(m)}_{\text{FNPF}}$. Dropping the subscript FNPF to avoid notational clutter, we have
\begin{align*}
\left(\sigma/\sqrt{m}\right)^{2(d-1)}\Lambda^{(m)}(s,s')&=\sigma^{2(d-1)}\frac{1}{m^{(d-1)}}\Pi_{l=1}^{d-1}\tau^{(m)}(s,s',l)\\
&=\sigma^{2(d-1)}\frac{1}{m^{(d-1)}}\Pi_{l=1}^{d-1}\left(m \tau(s,s',l)\right)\\
&=\sigma^{2(d-1)}\frac{1}{m^{(d-1)}}m^{(d-1)}\Pi_{l=1}^{d-1} \tau(s,s',l)\\
&=\sigma^{2(d-1)}\Pi_{l=1}^{d-1} \tau(s,s',l)\\
&=\sigma^{2(d-1)}\Lambda(s,s')
\end{align*}
\end{proof}

\section{DGN as a Lookup Table: Applying \Cref{th:main} to a pure memorisation task}\label{sec:mem}

In this section, we modify the DGN in \Cref{fig:dgn} into a memorisation network to solve a pure memorisation task. The objective of constructing the memorisation network is to understand the roles of depth and width in \Cref{th:main} in a simplified setting. In this setting, we show increasing depth till a point helps in training and increasing depth beyond it hurts training. 

\begin{definition}[Memorisation Network/Task]
Given a set of values $(y_s)_{s=1}^n\in  \R$, a memorisation network (with weights $\Theta\in\R^{\dnet}$) accepts $s\in[n]$ as its input and produces $\hat{y}_{\Theta}(s)\approx y_s$ as its output. The loss of the memorisation network is defined as $L_{\Theta}=\frac{1}{2}\sum_{s=1}^n (\hat{y}_{\Theta}(s)-y_s)^2$.
\end{definition}
\FloatBarrier
\begin{table}[h]
\centering
\begin{tabular}{| l |  l  |}\hline
Layer&  Memorisation Network\\\hline
Input  &$z_{\Theta}(0)=1$ \\
Pre-Activation & $q_{s,\Theta}(l)=\sum_{j}\Theta(i,j,l)\cdot z_{s,\Theta}(j,l-1)$\\
Hidden & $z_{s,\Theta}(i,l)=q_{s,\Theta}(i,l)\cdot G_{s}(i,l)$ \\
Final  Output & $\hat{y}_{\Theta}(s)=\sum_{j} \Theta(1,j,d) \cdot z_{s,\Theta}(j,d-1)$\\\hline
\end{tabular}
\caption{ Memorisation Network. The input is fixed and is equal to $1$. All the internal variables depend on the index $s$ and the parameter $\Theta$. The gating values $G_s(i,l)$ are external and independent variables.}
\label{tb:dgnmemo}
\end{table}

\textbf{Fixed Random Gating:} The memorisation network is described in \Cref{tb:dgnmemo}. In a memorisation network, the gates are \emph{fixed and random}, i.e., for each index $s\in[n]$, the gating values $G_{s}(i,l),\forall l\in[d-1], i\in[w] $ are sampled from $Ber(\mu), \mu\in(0,1)$ taking values in $\{0,1\}$,  and kept fixed throughout training. The input to the memorisation network is fixed as $1$, and since the gating is fixed and random there is a separate random sub-network to memorise each target $y_s\in\R$. The memorisation network can be used to memorise the targets  $(y_s)_{s=1}^n$ by training it using gradient descent by minimising the squared loss $L_{\Theta}$. In what follows, we let $K_0$ and $H_0$ to be the NTK and NPK of the memorisation network at initialisation.

\textbf{Performance of Memorisation Network:} From \Cref{prop:basic} we know that as $w\ra\infty$, the training error dynamics of the memorisation network follows:
\begin{align}
\dot{e}_t=-K_{0} e_t,
\end{align}
i.e., the spectral properties of $K_0$ (or $H_0$) dictates the rate of convergence of the training error to $0$. In the case of the memorisation network with fixed and random gates, we can calculate $\E{K_0}$ explicitly. 

\textbf{Spectrum of $H_0$:} The input Gram matrix $\Sigma$ is a $n\times n$ matrix with all entries equal to $1$ and its rank is equal to 1, and hence $H_0=\Lambda_0$. We can now calculate the properties of $\Lambda_0$. It is easy to check that $\mathbb{E}_{\mu}\left[\Lambda_0(s,s)\right]=(\mu w)^{(d-1)},\forall s\in[n]$ and $\mathbb{E}_{\mu}\left[\Lambda_0(s,s')\right]=(\mu^2 w)^{(d-1)},\forall s,s'\in[n]$.  For $\sigma=\sqrt{\frac{1}{\mu w}}$, and $\mathbb{E}_{\mu}\left[K_0(s,s)/d\right]=1$, and $\mathbb{E}_{\mu}\left[K_0(s,s')/d\right]=\mu^{(d-1)}$. 
\begin{figure}
\centering
\includegraphics[scale=0.3]{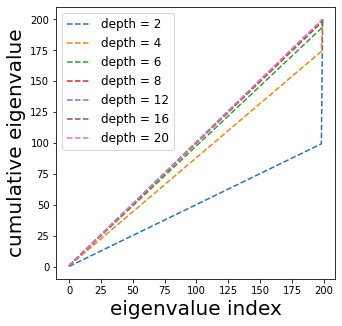}
\caption{Ideal spectrum of $\E{K_0}/d$ for a memorisation network for $n=200$.}
\label{fig:ideal-spectrum}
\end{figure}

\textbf{Why increasing depth till a point helps ?} 
We have:
\begin{align}\label{eq:mat}
\frac{\E{K_0}}{d}=\left[\begin{matrix}
1 &\mu^{d-1} &\ldots &\mu^{d-1} &\ldots\\ 
\ldots &1 &\ldots &\mu^{d-1} &\ldots\\ 
\ldots &\mu^{d-1} &\ldots &1 &\ldots \\
\ldots &\mu^{d-1} &\ldots &\mu^{d-1} &1\\ 
\end{matrix}\right]
\end{align}
i.e., all the diagonal entries are $1$ and non-diagonal entries are $\mu^{d-1}$. Now, let $\rho_i\geq 0,i \in [n]$ be the eigenvalues of $\frac{\E{K_0}}{d}$, and let $\rho_{\max}$ and $\rho_{\min}$ be the largest and smallest eigenvalues.  One can easily show that $\rho_{\max}=1+(n-1)\mu^{d-1}$ and corresponds to the eigenvector with all entries as $1$, and $\rho_{\min}=(1-\mu^{d-1})$ repeats $(n-1)$ times,  which corresponds to eigenvectors given by $[0, 0, \ldots, \underbrace{1, -1}_{\text{$i$ and $i+1$}}, 0,0,\ldots, 0]^\top \in \R^n$ for $i=1,\ldots,n-1$. Note that as $d\ra\infty$, $\rho_{\max},\rho_{\min}\ra 1$.

\textbf{Why increasing depth beyond a point hurts?} 
As the depth increases the variance of the entries $K_0(s,s')$ deviates from its expected value $\E{K_0(s,s')}$. Thus the structure of the Gram matrix degrades from \eqref{eq:mat}, leading to smaller eigenvalues.
\begin{figure}
\resizebox{\textwidth}{!}{
\begin{tabular}{cccc}
\includegraphics[scale=0.5]{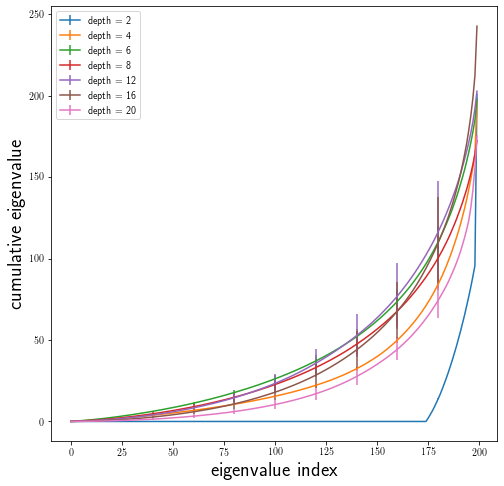}
&
\includegraphics[scale=0.5]{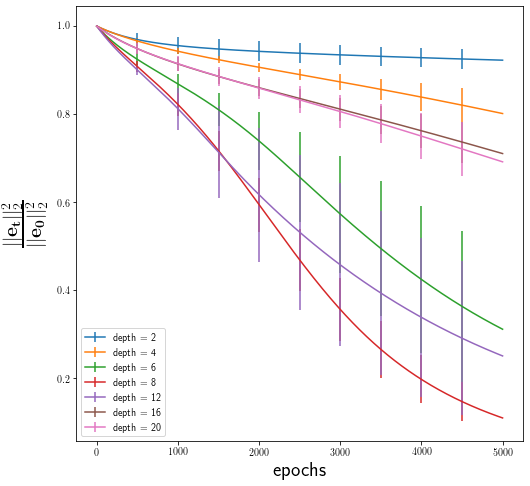}
&
\includegraphics[scale=0.5]{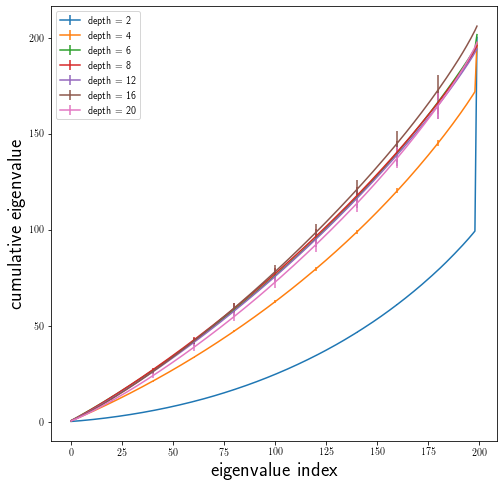}
&
\includegraphics[scale=0.5]{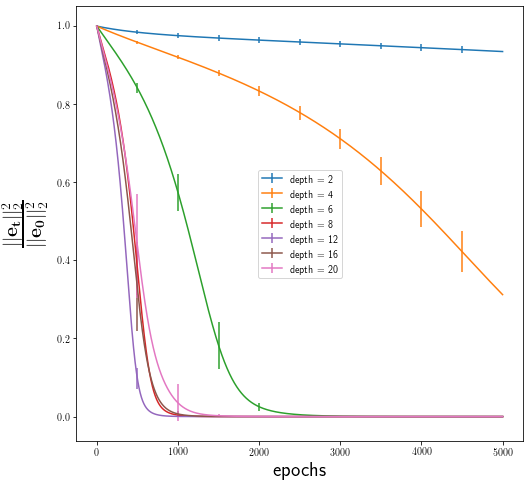}
\end{tabular}
}
\caption{Shows the plots for the memorisation network with $\mu=\frac{1}{2}$ and $\sigma=\sqrt{\frac{2}{w}}$. The number of points to be memorised is $n=200$. The left most plot shows the e.c.d.f for $w=25$ and the second plot from the left shows the error dynamics during training for $w=25$. The second plot from the right shows the e.c.d.f for $w=500$ and the right most plot shows the error dynamics during training for $w=500$. All plots are averaged over $10$ runs.}
\label{fig:dgn-frg-gram-ecdf}
\end{figure}

\subsection{Experiment}
We set $n=200$, and $y_s\sim\text{Uniform}[-1,1]$. We look at the cumulative eigenvalue (e.c.d.f) obtained by first sorting the eigenvalues in ascending order then looking at their cumulative sum. The ideal behaviour (\Cref{fig:ideal-spectrum}) as predicted from theory is that for indices $k\in[n-1]$, the e.c.d.f should increase at a linear rate, i.e., the cumulative sum of the first $k$ indices is equal to $k(1-\mu^{d-1})$, and the difference between the last two indices is $1+(n-1)\mu^{d-1}$. In \Cref{fig:dgn-frg-gram-ecdf}, we plot the actual e.c.d.f for various depths $d=2,4,6,8,12,16,20$ and $w=25,500$ (first and third plots from the left in \Cref{fig:dgn-frg-gram-ecdf}). 

\textbf{Roles of depth and width:} In order to compare how the rate of convergence varies with the depth, we set the step-size $\alpha=\frac{0.1}{\rho_{\max}}$, $w=100$. We use the vanilla SGD-optimiser. Note the$ \frac{1}{\rho_{\max}}$ in the stepsize, ensures that the uniformity of maximum eigenvalue across all the instances, and the convergence should be limited by the smaller eigenvalues. We also look at the convergence rate of the ratio $\frac{\norm{e_t}^2_2}{\norm{e_0}^2_2}$. We notice that for $w=25$, increasing depth till $d=8$ improves the convergence, however increasing beyond $d=8$ worsens the convergence rate. For $w=500$, increasing the depth till $d=12$ improves convergence, and $d=16,20$ are worse than $d=12$.  
\end{appendix}

%% file: varproof.tex
\textbf{Notation:} For $x,x'\in\R^{\din}$, let $\phi(p)=\phi_{x,\Tf_0}(p)$, and $\phi'(p)=\phi_{x',\Tf_0}(p)$. Also in what follows we use $\ta,\tb$ to denote the individual weights in the value network, and $p_a,p'_a,p_b,p'_b\in[P]$ to denote the paths. Further, unless otherwise specified, quantities $\ta,\tb,p_a,p'_a,p_b,p'_b$ are unrestricted.

\begin{proposition}
\begin{align}
\label{eq:k}K^\text{v}_{\Tdgn_0}(x,x')=&\sum_{\ta,p_a,p'_a} \phi(p_a) \phi'(p'_a) \partial_{\ta}v_{\Tv_0}(p_a) \partial_{\ta}v_{\Tv_0}(p'_a) \\
\end{align}
\end{proposition}
\begin{proof}
\begin{align}
K^\text{v}_{\Tdgn_0}(x,x')=&\ip{\nabla_{\Tv} \hat{y}_{\Tdgn_0}(x), \nabla_{\Tv}\hat{y}_{\Tdgn_0}(x')}\\
=&\sum_{\ta\in\Tv}\left(\sum_{p_a\in[P]}\phi(p_a) \partial_{\ta}v_{\Tv_0}(p_a) \right)\left(\sum_{p'_a\in[P]}\phi'(p'_a) \partial_{\ta}v_{\Tv_0}(p'_a) \right)\\
 	=& \sum_{\ta,p_a,p'_a} \phi(p_a) \phi'(p'_a) \partial_{\ta}v_{\Tv_0}(p_a) \partial_{\ta}v_{\Tv_0}(p'_a) 
\end{align}
\end{proof}

 We now drop $\Tv$ in $v_{\Tv_0}$ and $\text{v}$, $\Tdgn_0$ from $K^{\text{v}}_{\Tdgn_0}$, and we denote $\Tv$ by $\Theta$.

\begin{proposition}
\begin{align}
\label{eq:expk}\E{K(x,x')} 	=& \sum_{\ta,p_a} \phi(p_a) \phi'(p_a) \E{\left(\partial_{\ta}v(p_a)\right)^2}\\
\label{eq:expksquare}\E{K^2(x,x')}=& \underset {\tb,p_b,p'_b}{\underset{\ta,p_a,p'_a} {\sum}} \phi(p_a) \phi'(p'_a) \phi(p_b) \phi'(p'_b)\E{ \pta v(p_a) \pta v(p'_a) \ptb v(p_b) \ptb v(p'_b)}
\end{align}
\end{proposition}
\begin{proof}
\begin{align}
\E{K(x,x')} = &\sum_{\ta,p_a,p'_a} \phi(p_a) \phi'(p'_a)\E{\partial_{\ta}v(p_a) \partial_{\ta}v(p'_a)}\\
               \stackrel{(a)}{=}&\sum_{\ta,p_a} \phi(p_a) \phi'(p_a) \E{\left(\partial_{\ta}v(p_a)\right)^2}
\end{align}
where $(a)$ follows from \Cref{lm:dot} that for $p_a\neq p'_a$ $\E{\partial_{\ta}v(p_a) \partial_{\ta}v(p'_a)}=0$.

The expression for $\E{K^2(x,x')}$ is obtained by squaring the expression in \eqref{eq:k} and pushing the $\E{\cdot}$ inside the summation.

\end{proof}

\begin{definition}\label{def:bunching}\quad
\begin{enumerate}
\item Let $\tau=\left(p_a,p'_a,p_b,p'_b;\ta,\tb\right)$ denote the index used to sum the terms in the expression for $\E{K^2(x,x')}$ given in \eqref{eq:expksquare}. Note that the index contains $4$ path variables namely $p_a,p'_a,p_b,p'_b$ and $2$ weight variables namely $\ta,\tb$. 
\item An index $\tau=\left(p_a,p'_a,p_b,p'_b;\ta,\tb\right)$ is said to correspond to a `base' term if $p_a=p'_a$ and $p_b=p'_b$. We define $\B$ to be the set of indices corresponding to `base' terms given by $\B=\{\tau=\left(p_a,p'_a,p_b,p'_b;\ta,\tb\right)\colon p_a=p'_a, p_b=p'_b\}$.
\item For $\tau=\left(p_a,p'_a,p_b,p'_b;\ta,\tb\right)$, define $\omega(\tau) \eqdef \phi(p_a) \phi'(p'_a) \phi(p_b)\phi'(p'_b)$.
\item For $\tau=\left(p_a,p'_a,p_b,p'_b;\ta,\tb\right)$, define $E(\tau)\eqdef\E{\pta v(p_a) \pta v(p'_a) \ptb v(p_b) \ptb v(p'_b) }$.
\end{enumerate}
\end{definition}
\textbf{Remark:} \Cref{def:bunching} helps us to re-write \eqref{eq:expksquare} as $\E{K^2(x,x')}=\sum_{\tau} \omega(\tau)E(\tau)$.

\begin{proposition} \label{prop:prod}
For $\tau=(p_a,p'_a,p_b,p'_b;\ta,\tb)$, let $\ta$ belong to layer $l_a\in[d]$ and $\tb$ belong to layer $l_b\in[d]$. Let paths $p_a$ and $p'_a$ pass through $\ta$ and paths $p_b$ and $p'_b$ pass through $\tb$. 
\small
\begin{align}
E(\tau) =&\underbrace{\underset{l\neq l_a}{\underset{l\neq l_b} {\underset{l=1}{\Pi^{d}}}}  \E{\Theta(\I_l(p_a),\I_{l-1}(p_a),l )\Theta(\I_l(p'_a),\I_{l-1}(p'_a),l )\Theta(\I_l(p_b),\I_{l-1}(p_b),l )\Theta(\I_l(p'_b),\I_{l-1}(p'_b),l )}}_{\text{Term-I}}\nn\\
		      	&\underbrace{\cdot\E{ \Theta(\I_{l_b}(p_a),\I_{l_b-1}(p_a),l_b )\Theta(\I_{l_b}(p'_a),\I_{l_b-1}(p'_a),l_b) }}_{\text{Term-II}}\nn\\
\label{eq:etau}		      	&\underbrace{\cdot \E{\Theta(\I_{l_a}(p_b),\I_{l_a-1}(p_b),l_a )\Theta(\I_{l_a}(p'_b),\I_{l_a-1}(p'_b),l_a)}}_{\text{Term-III}}
\end{align}
\end{proposition}
\begin{proof}
\normalsize
Since the paths $p_a,p'_a$ pass through $\ta$ and paths $p_b,p'_b$ pass through $\tb$, it follows that $\pta v(p_a)\neq 0$, $\pta v(p'_a)\neq 0$,  $\ptb(p_b)\neq 0$ and $\ptb(p'_b)\neq 0$.  Note that, 
\begin{align*}
\pta v(p_a) =& \Pi_{l=1,l\neq l_a}^{d} \Theta(\I_{l}(p_a),\I_{l-1}(p_a),l )\\
\pta v(p'_a) =& \Pi_{l=1,l\neq l_a}^{d} \Theta(\I_{l}(p'_a),\I_{l-1}(p'_a),l )\\
\pta v(p_b) =& \Pi_{l=1,l\neq l_b}^{d} \Theta(\I_{l}(p_b),\I_{l-1}(p_b),l )\\
\pta v(p'_b) =& \Pi_{l=1,l\neq l_b}^{d} \Theta(\I_{l}(p'_b),\I_{l-1}(p'_b),l )\\
\end{align*}
The proof is complete by using the fact that weights of different layers are independent and pushing the $\mathbb{E}$ operator inside the $\E{ \pta v(p_a) \pta v(p'_a) \ptb v(p_b) \ptb v(p'_b)}$ to convert the expectation of products into a product of expectations.
\end{proof}

\begin{proposition}\label{lm:necc}
For $\tau=(p_a,p'_a,p_b,p'_b;\ta,\tb)$, $E(\tau)= \sigma^{4(d-1)}$ if and only if 

$\bullet$ \textbf{Condition I:} $p_a, p'_a$ pass through $\ta$ \emph{and} $p_b, p'_b$ pass through $\tb$.

$\bullet$ \textbf{Condition II:} In every layer, $l\in[d]$ either all the $4$ paths $p_a,p'_a,p_b,p'_b$ pass through the same weight \emph{or} there exists two distinct weights, say $\theta_{\text{I},l}$ and $\theta_{\text{II},l}$ such that, $2$ paths out of $p_a,p'_a,p_b,p'_b$ pass through $\theta_{\text{I},l}$ and the other $2$ paths pass through $\theta_{\text{II},l}$.

\end{proposition}

\begin{proof}\quad

\textbf{Sufficiency:} If \textbf{Condition I} and \textbf{Condition II} hold, then from  \eqref{eq:etau} it follows that $E(\tau)=\sigma^{4(d-2)}\cdot\sigma^2\cdot\sigma^2=\sigma^{4(d-1)}$.

\textbf{Necessity:} If \textbf{Condition I} does not hold, then either one of $\pta v(p_a), \pta v(p'_a), \ptb v(p_b), \ptb v(p'_b)$ becomes $0$. If \textbf{Condition II} does not hold, either Term-I or Term-II or Term-III in \eqref{eq:etau} evaluates to $0$ because all the weights involved are independent symmetric Bernoulli.

\end{proof}

\begin{definition}[Crossing]
Paths $\rho_a$ and $\rho_b$ are said to cross each other if they pass through the same node in one or one or more of the intermediate layers $l=2,\ldots,d-1$. For the sake of consistency, for paths $\rho_a$ and $\rho_b$ that do not cross, we call them to have $0$ crossings. 
\end{definition}

\begin{definition}[Splicing]
Let $(\rho_a,\rho_a,\rho_b,\rho_b)$ be $4$ paths (from a base term) occurring in pairs of $2$ each. Let $\rho_a$ and $\rho_b$ cross at $k\in\{0,\ldots,d-1\}$ intermediate nodes, belonging to layers $l_1,\ldots, l_k$ (let $l_0=0$ and $l_{k+1}=d$). Let the set of permutations of $(a,a,b,b)$ be denoted by $Pm\left((a,a,b,b)\right)\subset\{a,b\}^4$ . We say that paths $(p_a,p'_a,p_b,p'_b)$ to be `splicing' of $(\rho_a,\rho_a,\rho_b,\rho_b)$ if there exists $base(i,\cdot)\in Pm\left((a,a,b,b)\right),i=1,\ldots,k+1$ such that 
\begin{align*}
I_l(p_a) &= I_l(\rho_{base(i,1)}), l\in[l_{i-1},l_i], i=1,\ldots, k+1\\
I_l(p'_a) &= I_l(\rho_{base(i,2)}), l\in[l_{i-1},l_i], i=1,\ldots, k+1\\
I_l(p_b) &= I_l(\rho_{base(i,3)}), l\in[l_{i-1},l_i], i=1,\ldots, k+1\\
I_l(p'_b) &= I_l(\rho_{base(i,4)}), l\in[l_{i-1},l_i], i=1,\ldots, k+1
\end{align*}
\end{definition}

\begin{lemma}\label{lm:base}
Let $\tau=\left(p_a,p'_a,p_b,p'_b;\ta,\tb\right)$ be such that $E(\tau)=\sigma^{4(d-1)}$. Then there exists $\rho_a$ and $\rho_b$ such that $\rho_a$ passes through $\ta$, and $\rho_b$ passes through $\tb$, and $(p_a,p'_a,p_b,p'_b)$ is a splicing of $(\rho_a,\rho_a,\rho_b,\rho_b)$.
\end{lemma}
\begin{proof}
Using \Cref{lm:necc} and the fact that $p_a,p'_a,p_b,p'_b$ are paths, only the layouts shown in \Cref{fig:path-layout} are possible. In \Cref{fig:path-layout}, the $4$ different coloured lines stand for the $4$ different paths namely $p_a,p'_a,p_b,p'_b$. The hidden nodes are denoted by the circles. Here, \textbf{(a)} is the case where all the $4$ paths pass through the same weight in a given layer. \textbf{(b),(c), (d)} are the cases where $2$ paths out of the $4$ paths pass through one weight and the other $2$ paths pass through a different weight in a given layer.  \Cref{tb:layout-cond} provides the conditions for the possible current and next layer layouts.
\FloatBarrier
\begin{figure}[h]
\centering
\resizebox{\columnwidth}{!}{
\includegraphics[scale=0.1]{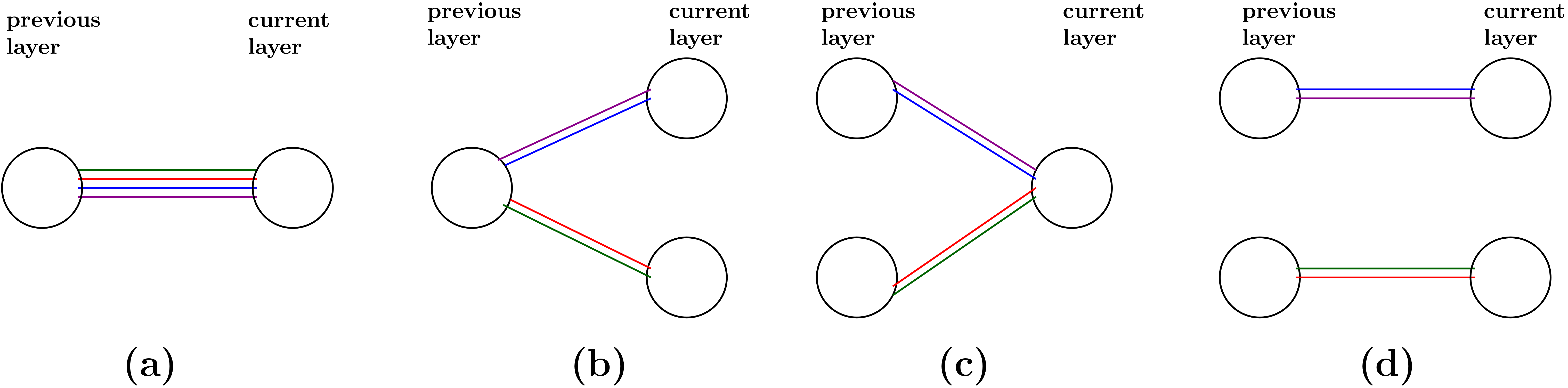}
}
\caption{Various ways in which the $4$  paths $p_a,p'_a,p_b,p'_b$ can pass through a given layer.}
\label{fig:path-layout}
\end{figure}
\FloatBarrier
\begin{table}[h]
\centering
\begin{tabular}{|c|c|}\hline
Current Layer Layout & Next Layer Layout\\\hline
\textbf{(a)} &  \textbf{(a)} or  \textbf{(b)}\\\hline
\textbf{(b)} &  \textbf{(c)} or  \textbf{(d)}\\\hline 
\textbf{(c)} &  \textbf{(a)} or  \textbf{(b)}\\\hline 
\textbf{(d)} &  \textbf{(c)} or  \textbf{(d)}\\\hline 
\end{tabular}
\caption{Show the possible current and next layer layouts.}
\label{tb:layout-cond}
\end{table}
Thus in each layer $p_a,p'_a,p_b,p'_b$ can always be paired to obtain $\rho_a$ and $\rho_b$. In the splicing, $base(i,1)$ specifies whether $p_a$ follows $\rho_a$ or $\rho_b$ between layers $l_{i-1}$ and $l_i$ (i.e., between crossing points). The role of $base(i,2)$, $base(i,3)$ and $base(i,4)$ can be explained in a similar manner.
\end{proof}

\begin{lemma}\label{lm:identitysplice}
Let $\tau'=\left(\rho_a,\rho_a,\rho_b,\rho_b;\ta,\tb\right)\in \B$ be an index in the base set such that $\rho_a$ and $\rho_b$ do not cross and $E(\tau')\neq 0$. Let $\tau=(p_a,p'_a,p_b,p'_b;\ta,\tb)$ be such that $(p_a,p'_a,p'_b,p_b)\neq\left(\rho_a,\rho_a,\rho_b,\rho_b\right) $ is a `splicing' of $(\rho_a,\rho_a,\rho_b,\rho_b)$. Then $E(\tau)=0$.
\end{lemma}

\begin{proof}
Since $E(\tau')\neq 0$, it follows that $\rho_a$ passes through $\ta$ and $\rho_b$ passes through $\tb$. Since $\rho_a$ and $\rho_b$ do not cross each other, the only possible splicings are the permutations of $(\rho_a,\rho_a,\rho_b,\rho_b)$ itself. For the sake of concreteness, let us pick a $\tau$ such that $(p_a,p'_a,p'_b,p_b)=(\rho_a,\rho_b,\rho_a,\rho_b)$ (a non-identity permutation). For $E(\tau)\neq 0$ to hold, $\pta v(\rho_b)\neq0 $ and $\ptb v(\rho_a) \neq 0$ should also hold, which implies both $\rho_a$ and $\rho_b$ pass through $\ta$ and $\tb$. However, we assumed that $\rho_a$ and $\rho_b$ do not cross each other. Hence,  $E(\tau)= 0$ for any $\tau$ such that $(p_a,p'_a,p'_b,p_b)$ is a non-identity permutation of $(\rho_a,\rho_a,\rho_b,\rho_b)$.
\end{proof}

\begin{proposition}
Let $\tau$ and $\B$ be as in \Cref{def:bunching}, then 
\begin{align*}
\E{K(x,x')}^2= \sum_{\tau\in\B}\omega(\tau)E(\tau)
\end{align*}
\end{proposition}
\begin{proof}
Writing down the left-hand and right-hand sides, we have:
\begin{align*}
\E{K(x,x')}^2 =& \underset {\tb,p_b}{\underset{\ta,p_a}{\sum}} \phi(p_a) \phi'(p_a) \phi(p_b) \phi'(p_b) \E{\left(\pta v(p_a)\right)^2}\E{\left(\ptb v(p_b)\right)^2}\\
\sum_{\tau\in\B}\omega(\tau)E(\tau) =& \underset {\tb,p_b}{\underset{\ta,p_a} {\sum}} \phi(p_a) \phi'(p_a) \phi(p_b) \phi'(p_b)\E{ \left(\pta v(p_a) \right)^2\left( \ptb v(p_b)\right)^2}
\end{align*}
When $\pta v(p_a)\neq 0$ and $\ptb v(p_b)\neq 0$, for symmetric Bernoulli weights it follows that $\E{\left(\pta v(p_a)\right)^2}\E{\left(\ptb v(p_b)\right)^2}= \E{ \left(\pta v(p_a) \right)^2\left( \ptb v(p_b)\right)^2}=\sigma^{4(d-1)}$.
\end{proof}

\begin{theorem}\label{th:var}
Let the weights be chosen as in \Cref{th:main}. Then, it follows that
 \begin{align*} 
 Var\left[K(x,x')\right]\leq C\din^2\frac{d^3}{w}
 \end{align*}
\end{theorem}
\begin{proof}
\begin{align*}
Var\left[K(x,x')\right] =&\E{K^2(x,x')}-\E{K(x,x')}^2\\
=&\sum_{\tau} \omega(\tau)E(\tau)-\sum_{\tau\in\B}\omega(\tau)E(\tau)\\
=&\sum_{\tau\notin\B}\omega(\tau)E(\tau)
\end{align*}
In what follows, without loss of generality, let $|\omega(\tau)|\leq1$. Then,  
\begin{align*}
Var\left[K(x,x')\right] \leq \sum_{\tau\notin\B}E(\tau)
\end{align*}
Let $\bar{\B}=\{\tau\notin\B\}$. From \Cref{lm:base} we know that every $\tau\in\bar{\B}$ such that $E(\tau)=\sigma^{4(d-1)}$ can always be identified with a base term $\tau'=(\rho_a,\rho_a,\rho_b,\rho_b;\ta,\tb)\in\B$, and from \Cref{lm:identitysplice}, we know that in such a $\tau'$, the base paths $\rho_a$ and $\rho_b$ cross $k>0$ times. Now, there are $\binom{(d-1)}{k}< d^k$ possible ways in which the $k$ crossing can occur within the $(d-1)$ layers, and within each layer there are $w$ possible nodes in which such crossings can occur. The total number of paths that pass through $k<d-1$ nodes is $\frac{P}{w^k}$, where $P=\din w^{(d-1)}$. And the number of splicings of base terms with $k$ crossings is less than $6^{k+1}$. Once we obtain the paths, the crossings, the splicing, the weights $\ta$ and $\tb$ can each occur in up to any of the $d-1$ layers.
Putting all this together, we have
\begin{align*}
Var\left[K(x,x')\right] \leq& \sum_{k=1}^\infty d^26^{k+1}\cdot(wd)^k\cdot\left(\frac{P^2}{w^{2k}}\right)\sigma^{4(d-1)}\\
\leq& 6\din^2 {\sigma'}^2\left(\frac{6d^3}{w}\right)\left(\frac{1}{1-\frac{6d}{w}}\right)\\
\leq& C\din^2\frac{d^3}{w}
\end{align*}
\end{proof}

%% file: neural-path-features.bbl
\begin{thebibliography}{6}
\providecommand{\natexlab}[1]{#1}
\providecommand{\url}[1]{\texttt{#1}}
\expandafter\ifx\csname urlstyle\endcsname\relax
  \providecommand{\doi}[1]{doi: #1}\else
  \providecommand{\doi}{doi: \begingroup \urlstyle{rm}\Url}\fi

\bibitem[Arora et~al.(2019)Arora, Du, Hu, Li, Salakhutdinov, and
  Wang]{arora2019exact}
Sanjeev Arora, Simon~S Du, Wei Hu, Zhiyuan Li, Russ~R Salakhutdinov, and
  Ruosong Wang.
\newblock On exact computation with an infinitely wide neural net.
\newblock In \emph{Advances in Neural Information Processing Systems}, pages
  8139--8148, 2019.

\bibitem[Cao and Gu(2019)]{cao2019generalization}
Yuan Cao and Quanquan Gu.
\newblock Generalization bounds of stochastic gradient descent for wide and
  deep neural networks.
\newblock In \emph{Advances in Neural Information Processing Systems}, pages
  10835--10845, 2019.

\bibitem[Du and Hu(2019)]{dudln}
Simon~S Du and Wei Hu.
\newblock Width provably matters in optimization for deep linear neural
  networks.
\newblock \emph{arXiv preprint arXiv:1901.08572}, 2019.

\bibitem[Du et~al.(2018)Du, Lee, Li, Wang, and Zhai]{dudnn}
Simon~S Du, Jason~D Lee, Haochuan Li, Liwei Wang, and Xiyu Zhai.
\newblock Gradient descent finds global minima of deep neural networks.
\newblock \emph{arXiv preprint arXiv:1811.03804}, 2018.

\bibitem[Jacot et~al.(2018)Jacot, Gabriel, and Hongler]{ntk}
Arthur Jacot, Franck Gabriel, and Cl{\'e}ment Hongler.
\newblock Neural tangent kernel: Convergence and generalization in neural
  networks.
\newblock In \emph{Advances in neural information processing systems}, pages
  8571--8580, 2018.

\bibitem[Shamir(2018)]{shamir}
Ohad Shamir.
\newblock Exponential convergence time of gradient descent for one-dimensional
  deep linear neural networks.
\newblock \emph{arXiv preprint arXiv:1809.08587}, 2018.

\end{thebibliography}


\begin{thebibliography}{14}
\providecommand{\natexlab}[1]{#1}
\providecommand{\url}[1]{\texttt{#1}}
\expandafter\ifx\csname urlstyle\endcsname\relax
  \providecommand{\doi}[1]{doi: #1}\else
  \providecommand{\doi}{doi: \begingroup \urlstyle{rm}\Url}\fi

\bibitem[Arora et~al.(2019a)Arora, Du, Hu, Li, Salakhutdinov, and
  Wang]{arora2019exact}
Sanjeev Arora, Simon~S Du, Wei Hu, Zhiyuan Li, Russ~R Salakhutdinov, and
  Ruosong Wang.
\newblock On exact computation with an infinitely wide neural net.
\newblock In \emph{Advances in Neural Information Processing Systems}, pages
  8139--8148, 2019a.

\bibitem[Balestriero et~al.(2018)]{balestriero2018spline}
Randall Balestriero et~al.
\newblock A spline theory of deep learning.
\newblock In \emph{International Conference on Machine Learning}, pages
  374--383, 2018.

\bibitem[Cao and Gu(2019)]{cao2019generalization}
Yuan Cao and Quanquan Gu.
\newblock Generalization bounds of stochastic gradient descent for wide and
  deep neural networks.
\newblock In \emph{Advances in Neural Information Processing Systems}, pages
  10835--10845, 2019.

\bibitem[Du and Hu(2019)]{dudln}
Simon~S Du and Wei Hu.
\newblock Width provably matters in optimization for deep linear neural
  networks.
\newblock \emph{arXiv preprint arXiv:1901.08572}, 2019.

\bibitem[Du et~al.(2018)Du, Lee, Li, Wang, and Zhai]{dudnn}
Simon~S Du, Jason~D Lee, Haochuan Li, Liwei Wang, and Xiyu Zhai.
\newblock Gradient descent finds global minima of deep neural networks.
\newblock \emph{arXiv preprint arXiv:1811.03804}, 2018.

\bibitem[Fiat et~al.(2019a)Fiat, Malach, and Shalev{-}Shwartz]{sss}
Jonathan Fiat, Eran Malach, and Shai Shalev{-}Shwartz.
\newblock Decoupling gating from linearity.
\newblock \emph{CoRR}, abs/1906.05032, 2019a.
\newblock URL \url{http://arxiv.org/abs/1906.05032}.

\bibitem[Fiat et~al.(2019b)Fiat, Malach, and Shalev-Shwartz]{fiat}
Yonathan Fiat, Eran Malach, and Shai Shalev-Shwartz.
\newblock Decoupling gating from linearity, 2019b.
\newblock URL \url{https://openreview.net/forum?id=SJGyFiRqK7}.

\bibitem[Jacot et~al.(2018)Jacot, Gabriel, and Hongler]{ntk}
Arthur Jacot, Franck Gabriel, and Cl{\'e}ment Hongler.
\newblock Neural tangent kernel: Convergence and generalization in neural
  networks.
\newblock In \emph{Advances in neural information processing systems}, pages
  8571--8580, 2018.

\bibitem[Jacot et~al.(2019)Jacot, Gabriel, and Hongler]{jacot2019freeze}
Arthur Jacot, Franck Gabriel, and Cl{\'e}ment Hongler.
\newblock Freeze and chaos for dnns: An {NTK} view of batch normalization,
  checkerboard and boundary effects.
\newblock \emph{arXiv preprint arXiv:1907.05715}, 2019.

\bibitem[Kingma and Ba(2014)]{adam}
Diederik~P Kingma and Jimmy Ba.
\newblock Adam: A method for stochastic optimization.
\newblock \emph{arXiv preprint arXiv:1412.6980}, 2014.

\bibitem[Neyshabur et~al.(2015)Neyshabur, Salakhutdinov, and
  Srebro]{neyshabur2015path}
Behnam Neyshabur, Russ~R Salakhutdinov, and Nati Srebro.
\newblock Path-sgd: Path-normalized optimization in deep neural networks.
\newblock In \emph{Advances in Neural Information Processing Systems}, pages
  2422--2430, 2015.

\bibitem[Saxe et~al.(2013)Saxe, McClelland, and Ganguli]{ganguli}
Andrew~M Saxe, James~L McClelland, and Surya Ganguli.
\newblock Exact solutions to the nonlinear dynamics of learning in deep linear
  neural networks.
\newblock \emph{arXiv preprint arXiv:1312.6120}, 2013.

\bibitem[Shamir(2019)]{shamir}
Ohad Shamir.
\newblock Exponential convergence time of gradient descent for one-dimensional
  deep linear neural networks.
\newblock In \emph{Conference on Learning Theory}, pages 2691--2713, 2019.

\bibitem[Srivastava et~al.(2014)Srivastava, Masci, Gomez, and
  Schmidhuber]{srivastava2014understanding}
Rupesh~Kumar Srivastava, Jonathan Masci, Faustino Gomez, and J{\"u}rgen
  Schmidhuber.
\newblock Understanding locally competitive networks.
\newblock \emph{arXiv preprint arXiv:1410.1165}, 2014.

\end{thebibliography}
